\documentclass[nohyperref]{article}

\usepackage{microtype}
\usepackage{graphicx}
\usepackage{subfigure}
\usepackage{booktabs} %

\usepackage{hyperref}

\usepackage[accepted]{icml}

\usepackage{amsmath}
\usepackage{amssymb}
\usepackage{mathtools}
\usepackage{amsthm}
\usepackage{bm}
\usepackage{dsfont}

\usepackage[capitalize,noabbrev]{cleveref}

\theoremstyle{plain}
\newtheorem{theorem}{Theorem}[section]

\newtheorem{corollary}[theorem]{Corollary}
\theoremstyle{definition}
\newtheorem{definition}[theorem]{Definition}
\newtheorem{assumption}[theorem]{Assumption}
\theoremstyle{remark}

\newtheorem*{theorem*}{Statement}

\newcommand{\norm}[1]{\left\lVert#1\right\rVert}

\usepackage[textsize=tiny]{todonotes}

\icmltitlerunning{Convolutional and Residual Networks Provably Contain Lottery Tickets}

\begin{document}

\twocolumn[
\icmltitle{Convolutional and Residual Networks Provably Contain Lottery Tickets}

\icmlsetsymbol{equal}{*}

\begin{icmlauthorlist}
\icmlauthor{Rebekka Burkholz}{cispa}
\end{icmlauthorlist}

\icmlaffiliation{cispa}{CISPA Helmholtz Center for Information Security, Saarbr\"ucken, Germany}

\icmlcorrespondingauthor{Rebekka Burkholz}{burkholz@cispa.de}

\icmlkeywords{theory, deep learning, lottery ticket}

\vskip 0.3in
]

\printAffiliationsAndNotice{\icmlEqualContribution} %

\begin{abstract}
The Lottery Ticket Hypothesis continues to have a profound practical impact on the quest for small scale deep neural networks that solve modern deep learning tasks at competitive performance. These lottery tickets are identified by pruning large randomly initialized neural networks with architectures that are as diverse as their applications. Yet, theoretical insights that attest their existence have been mostly focused on deep fully-connected feed forward networks with ReLU activation functions. We prove that also modern architectures consisting of convolutional and residual layers that can be equipped with almost arbitrary activation functions can contain lottery tickets with high probability. 
\end{abstract}

\section{Introduction}
The Lottery ticket (LT) Hypothesis \citep{frankle2019lottery} has fueled the interest in deep neural network pruning to a reduce the number of trainable parameters with the purpose to save computational resources, regularize, and perform meaningful structure learning. 
Most newly developed algorithms are benchmarked and evaluated in the imaging domain.
Naturally, most architectures that are pruned in practice contain therefore convolutional layers.
In particular residual blocks, and skip connections in general, seem to provide iterative pruning algorithms an advantage over less computationally cumbersome approaches \citep{sanity2}.
It is subject of an ongoing debate to which degree different algorithms are successful in finding task specific computational neural network structures \citep{sanity,sanity2,plant} and whether LTs are identifiable by contemporary pruning algorithms to solve complex problems with large scale architectures \citep{rewind,rewindVsFinetune}. 
Theoretical insights into the conditions when we can expect to find LTs can provide guidance regarding when improvements could be feasible. %

Our work contributes to the discussion with the assurance that LTs likely exist under realistic conditions in convolutional networks with or without residual blocks even when pruning algorithms are currently challenged to find them.
This is in line with the Strong Lottery Ticket Hypothesis (SLTH), which has been posed by \citep{ramanujan2019whats} based on experiments in inspiration of \cite{zhou2019deconstructing}.
It suggests that a sufficiently large neural network with random parameters contains, with high probability, for each target network (of certain maximal size) a sub-network that can approximate the target network with high accuracy.
Such a sub-network is also called strong LT and does not need to be trained in order to achieve a performance that is competitive with the one of the target network.

The existence of such strong LTs has been proven for fully-connected feed forward architectures and ReLU activation functions by providing a probabilistic lower bound on the required width of the larger random network.
First proven by \cite{malach2020proving}, the width requirements have succinctly been improved to a logarithmic factor in the relevant variables \cite{pensia2020optimal,orseau2020logarithmic} and extended to nonzero biases \citep{nonzerobiases}.
The only results for convolutional architectures are provided by \citep{uniExist,cnnexist} but either apply to specific targets \cite{uniExist} or are restricted to positive inputs \citep{cnnexist} and do not cover common data transformations. 
Furthermore, all of these results rely on random networks that have at least twice the depth of the target network.
This excludes constructions, in which the LTs can have residual blocks of the same size as in the target network, and reduces the expressiveness of the target networks, as they cannot utilize a large part of the available depth for a sparser representation \citep{approx,deepOverShallow}.
To overcome these limitations, we follow a similar strategy as \citep{depthexist} to extend existence results to a larger class of activation functions and random networks that have a similar depth as the target network ($L+1$).
We solve the additional challenge to prune and construct convolutional filters with and without skip connections and residual blocks.
In doing so, we propose a different construction than \citep{cnnexist} that is not restricted to positive inputs.

\subsection{Contributions}
1) We prove the existence of strong lottery tickets in convolutional neural network architectures, potentially with skip connections and residual blocks.  2) Our constructions are not restricted to CNNs with positive inputs and ReLUs in contrast to \citep{cnnexist}.
3) Our proofs apply to a large class of activation functions, including ReLUs, Leaky ReLUs, Tanh, and Sigmoids.
4) We present two types of constructions: (a) one in which the large, randomly initialized neural network that contains LTs has at least twice the depth of a target network, i.e. $L_0 = 2L_t$; and one in which the target can leverage almost the full depth of the large network for a sparser representation, as $L_0 = L_t+1$. 
5) We verify in experiments that our theory derives realistic conditions.
Based on insights on solving subset sum approximation problems experimentally, we assess the expected sparsity of our LTs.

\subsection{Related Literature}
Most LT experiments are conducted in the context of image classification and thus rely heavily on pruning convolutional and residual neural network architectures to reduce the number of trainable parameters of a neural network \citep{braindamage,trimming,IMPfirst,frankle2019lottery,srinivas2016gendropout,orthoRepair,earlybird,rewind,rewindVsFinetune,weightcor,liu2021:finetune,weightelim,LTreg,elasticLTH,sigmoidl0,lecun1990optimal,hassibi1992second,dong2017surgeon,li2017pruneconv,molchanov2017pruneinf,validateManifold}. 
Some exceptions include graph neural networks \citep{graphLTH} and GANs \citep{chen2021dataefficient}, which still utilize convolutions.
One of the main objectives is to reduce the computational burden associated with deep learning. 
This can also be achieved with the help of core sets \citep{lessdatamore} or by starting the pruning not from a dense but a sparse random architecture \citep{evci2019rigging,plastic}. 
Pruning before training \citep{grasp,snip,snipit,synflow,ramanujan2019whats} is also a promising research direction but iterative pruning methods often perform better \citep{frankle2021review,sanity2,plant}, while most benefits seem to result from residual skip connections \citep{sanity2}. 
Another objective in the identification of LTs is structure learning, which seems to be more effective at lower sparsity levels \citep{sanity,orthoRepair} and in many cases Iterative Magnitude Prunings (IMP) \cite{IMPfirst,frankle2019lottery} can fail to find structures that perform superior to random or smaller dense networks \citep{sanity2}. 
Regardless, pruning can have provable regularization and generalization properties \citep{LTgeneralization}. 
At least for fully-connected architectures it has also been shown that structurally relevant LTs exist theoretically.

Most of the discussed pruning methods try to find weak LTs by identifying a sparse neural network architecture that is well trainable.
Strong LTs are sparse sub-networks that perform well even without training and just rely on their initial parameters \citep{zhou2019deconstructing, ramanujan2019whats} and are thus also weak LTs.
Their existence has been proven for fully-connected feed forward networks with \textsc{ReLU} activation functions by providing lower bounds on the width of the large, randomly initialized neural network that contains them \citep{malach2020proving,pensia2020optimal,orseau2020logarithmic,nonzerobiases,uniExist}.
In addition, it was shown that multiple candidate tickets exist that are also robust to parameter quantization \citep{multiprize}.
These works are restricted to ReLUs and always assume that the large randomly initialized neural networks has at least twice the depth of a target network $L_0 \geq 2L_t$. 
\citep{depthexist} extends these results to more general activation functions and also introduces a strategy to handle $L_0 \geq L_t+1$. 
Up to our knowledge, \citep{uniExist,cnnexist} are the only theoretical works on convolutional architectures and apply only to ReLU activation functions.  
While \citep{uniExist} uses convolutions to obtain specific representations of basis functions, \citep{cnnexist} studies general convolutional layers but is restricted to target networks with positive inputs, which does not cover the common image transformation procedures. 
We present more general results that also apply to potentially negative inputs, require less depth, can handle skip connections, which seem to be essential for the success of state-of-the-art pruning algorithms \citep{sanity2}, and cover a large class of activation functions.

\section{Background and notation}
Let a convolutional neural network $f: \mathcal{D} \subset \mathbb{R}^{c_0 \times d_0} \rightarrow \mathbb{R}^{c_L \times d_L}$ be defined on a compact domain $\mathcal{D}$ and have channels $\bar{c} = [c_0, c_1, ..., c_L]$, i.e., depth $L$ and width $c_l$ in layer $l \in [L] := \{0, ..., L\}$. 
It is equipped with a continuous activation function $\phi(x)$ that has Lipschitz constant $T$ on a compact domain that includes the possible inputs.
$f$ maps an input tensor $\bm{x}^{(0)}$ to neurons $x^{(l)}_{ik}$ as:
\begin{equation}\label{eq:DNN}
 \bm{x}^{(l)}_{i} = \phi\left(\bm{h}^{(l)}_{i} \right),  \ \ \  \bm{h}^{(l)}_{i} = \sum^{c_{l-1}}_{j=1} \bm{W}^{(l)}_{ij} * \bm{x}^{(l-1)}_{j} + b^{(l)}_i,
\end{equation}
where $\bm{h}^{(l)}$ is the pre-activation, $\bm{W}^{(l)} \in \mathbb{R}^{c_{l} \times c_{l-1} \times k_l}$ is the weight tensor that consists of filters (or convolutional kernels), $\bm{b}^{(l)} \in \mathbb{R}^{c_l}$ is the bias vector of layer $l$, and $*$ denotes a convolution operation.
To simplify and generalize our notation, we have flattened the filter dimension to $k_l$.
For 2d convolutions, as they are commonly in use on imaging data, the weight tensor would actually have the size $\bm{W}^{(l)} \in \mathbb{R}^{c_{l} \times c_{l-1} \times k'_{1,l} \times k'_{2,l}}$ so that $k_l =  k'_{1,l}  k'_{2,l}$.
The convolution operation between any 2-dimensional tensors $K$ and $X$ is defined as $(\bm{K} * \bm{X})_{ij} =  \sum_{i',j'} K_{i'j'} X_{(i-i'+1)(j-j'+1)}$ in this case. 
We assume that the inputs are always suitably padded with zeros and that the symbol $*$ performs the convolutions in the right dimensions.
The flattened notation just makes it easier to discuss higher dimensional filters at the same time. 
In addition to convolutional layers, we also allow for residual and more general skip connections. 
Skip connections modify the network above as
\begin{align}\label{eq:nnskip}
\bm{x}^{(l)}_{i} = \phi\left(\bm{h}^{(l)}_{i} \right) + \sum^{l-1}_{t=1} \sum^{c_{t}}_{j=1} \bm{M}^{(l,t)}_{ij} * \bm{x}^{(t)}_{j}.
\end{align}
Usually, most of the operators $\bm{M}^{(l,t)}_{ij}$ are zero.
In case of residual connections, $\bm{M}^{(l,t)}_{ij1}=1$ are one-dimensional filters that encode the identity and do not impose any additional learnable or prunable parameters. 
Without loss of generality, we assume that each parameter (weight or bias) $\theta$ is bounded by $|\theta| \leq 1-\epsilon$. 
In addition, we require that each tensor element is bounded as $|x_{iq}| \leq 1$.
Otherwise, our estimate of the error that we allow in the approximation of each target parameter would become more complicated.
Moreover, we denote with $N_{w,l}$ the number of all nonzero weight and bias parameters in Layer $l$ that do not correspond to skip connections, while $N_{m,l}$ counts the number of all nonzero parameters involved in skip connections that lead to Layer $l$. 

We distinguish three different types of neural networks: a target network $f_t$, a LT $f_{\epsilon}$, and a source network $f_0$.
The target network $f_t$ is approximated by the LT $f_{\epsilon}$, which we obtain by pruning $f_0$.
We also write $f_{\epsilon} \subset f_{0}$, meaning that $f_{\epsilon}$ is constructed by masking some parameters of $f_0$, i.e. setting some of them to $0$, while the other parameters keep their original value. 
The parameters of the source network $f_{0}$ are drawn from a random distribution as follows.
\begin{assumption}[Parameter initialization]\label{def:initconv}
We assume that the parameters of the source network $f_0$ are independently distributed as $w^{(l)}_{ij} \sim  U\left([-\sigma_l, \sigma_l]\right)$, $b^{(1)}_{i} \sim  U\left([-\sigma_l, \sigma_l]\right)$ and $b^{(l)}_{i} = 0$ for $l>1$.
\end{assumption}
Note that they could also follow any other distribution that contains a uniform distribution, for instance, a normal distribution \citep{pensia2020optimal}.
In our theorems and proofs, we choose $\sigma_l$ conveniently based on the activation functions.
For ReLUs, for instance, $\sigma_l=1$ is common. 
In practice, we usually have $\sigma_l \propto 1/\sqrt{c k}$ to avoid vanishing or exploding gradients. 
To transfer our results to this setting, we need to scale each parameter of the LT of our proofs by a scaling factor $\lambda_l$ \citep{depthexist}.
For homogeneous activation functions like ReLUs or Leaky ReLUs, these scaling factors can also be joined into a single one  $\lambda = \prod_l \lambda_l$ that is applied to the output \citep{nonzerobiases}. 
Note that if the target network parameter would not fulfill our assumption $|\theta| \leq 1$, we could simply adjust the scaling factors so that the initial parameter distribution and the network parameters vary within the same range.
As these scaling factors could still be learned or just derived based on the parameter initialization, our existence results transfer to realistic parameter initialization settings.

As the parameters of the source network are random, $f_0$ needs to be bigger than the target so that we have enough alternatives to pick the right ones. %
Because of two strategies that help us increase our options, the LT also consists of more parameters and neurons than the target network. 
These two strategies are (a) solving subset sum approximation problems and (b) using several layers to approximate a target layer.
\begin{figure*}
\includegraphics[width=\textwidth]{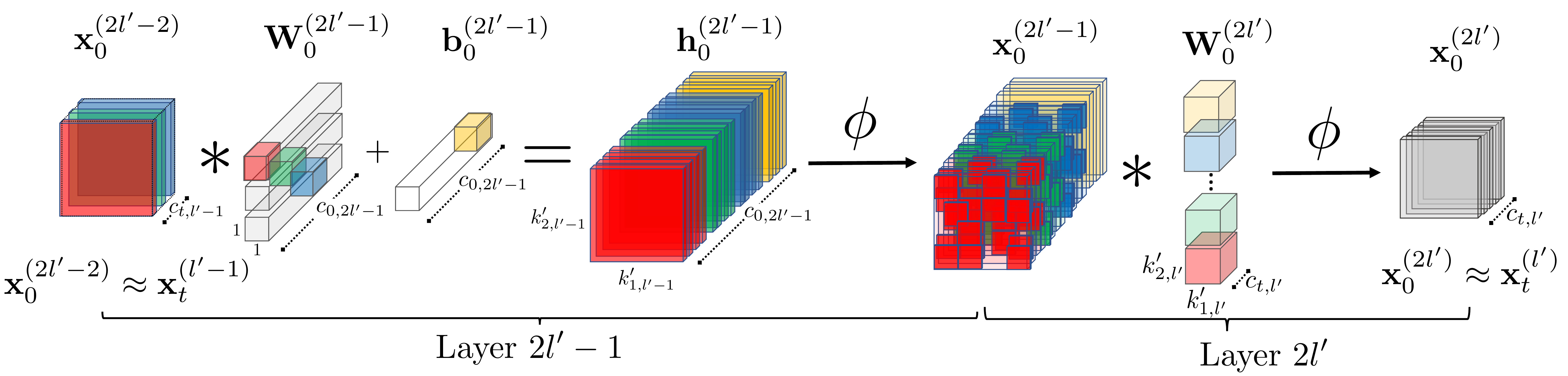}
\caption{Construction of the target Layer $l'$ in Layers $2l'-1$ and $2l'$ of the source network $f_0$.}
\label{fig:constr2l}
\end{figure*}
\paragraph{Subset Sum Approximation}
Instead of searching for a single parameter in $f_0$ that closely matches a target parameter $\theta_t$, we approximate it by the sum of multiple parameters $\theta_t \approx \sum_{n \in S} \theta_{0,n}$.
There are usually many options to represent such a sum and each possibility can be cast as a subset of a larger base set of size $m$, which contains $2^m$ candidate subsets.  
This explains why solving subset sum approximation problems to find a suitable subset is successful with high probability based on relatively small base sizes $m$.
We frequently use a theorem by \cite{subsetsum}, which has been extended by \citep{uniExist} to solve subset sum approximation problems if the random variables are not necessarily identically distributed.
For convenience it is stated as Cor.~\ref{thm:subsetsumExtended} in more general form in the appendix.
We primarily utilize the following simplification in our construction. 
\begin{corollary}[Subset sum approximation \citep{subsetsum}]\label{thm:subsetsum}
Let $X_1, ..., X_{m}$  be independent, uniformly distributed random variables with $X_k \sim U[-1,1]$ or $X_k \sim U[-1,1] U[-1,1]$ and $\epsilon, \delta \in (0,1)$ be given. 
Then for any $\theta_t \in [-1,1]$ there exists a subset $S \subset [m]$ so that with probability at least $1-\delta$ we have $|\theta_t - \sum_{k \in S} X_k| \leq \epsilon$ if $m \geq C \log\left(\frac{1}{\min\left(\delta, \epsilon\right)} \right)$.
\end{corollary}
Naturally, a question that decides about the practicality of using this result is the size of the constant $C$. 
Repeated solutions of subset sum problems by optimal exhaustive searches over subsets for different base set sizes $m$ with $X_k\sim U[-1,1]$ suggest that $C \approx 3$, which means that our approach is feasible.
More insights on subset sum problems are discussed in the experiments section.

(b) The second strategy to construct LTs is concerned with making base sets $\{X_1, ..., X_{m}\}$ available to solve subset sum approximation problems.
A common approach is to create multiple versions of the input to a layer, which usually populate an additional intermediary layer in $f_0$ and $f_{\epsilon}$ before approximate target neurons are constructed in the following layer.
This results in source networks and LTs that need twice the depth of a target network, i.e, $L_0 = 2 L_t$.
\cite{cnnexist} has transferred this idea for fully-connected feed forward architectures to convolutional layers with ReLU activation functions.
However, the first layer has multi-dimensional input in general, which limits the approach to positive inputs.
Our first contribution is to derive a different ($L_0 = 2L_t$)-construction that can handle any input and works for multiple activation functions, including ReLUs. 
Afterwards, we transfer the $L_t+1$ construction idea for fully-connected feed forward neural networks \citep{depthexist} and general activation functions to convolutional layers.
An important consequence of this construction is that we can also cover residual blocks that are of the same size in all three networks, the target, the source, and the LT.

\paragraph{Activation Functions}
The ($L_0 = 2L_t$)-construction in fully-connected networks relies on the fact that ReLUs $\phi_R(x) = \max\{x,0\}$ can easily represent the identity as $x = \phi_R(x)-\phi_R(-x)$ so that neurons in the intermediary layer correspond to the positive $\phi(x^{(l-1)}_{i})$ and the negative part $\phi(-x^{(l-1)}_{i})$ of input neurons $x^{(l-1)}_{i}$. 
Similarly, Leaky ReLUs $\phi_{LR}(x) = \phi_R(x) - \alpha \phi_R(-x)$ encode the identity as $x = (\phi_{LR}(x)-\phi_{LR}(-x))/(1+\alpha)$.
Most other activation functions can be approximated locally around the origin by a shifted Leaky ReLU and thus fulfill our following assumption.
\begin{assumption}[Activation function (first layer)]\label{def:act}
For any given $\epsilon' > 0$ exists a neighborhood $[-a(\epsilon'),a(\epsilon')]$ of $0$ with $a(\epsilon') >0$ so that the activation function $\phi$ can be approximated by $\widehat{\phi}(x)$ on that neighborhood such that $\sup_{x \in {[-a,a]}} | \phi(x) - \widehat{\phi}(x)| \leq \epsilon'$, where $\widehat{\phi}(x) = m_{+} x + d$ for $x \geq 0$ and $\widehat{\phi}(x) = m_{-} x + d$ for $x < 0$ with $m_{+}, m_{-}, d \in \mathbb{R}$ and $m_{+} + m_{-} \neq 0$. 
We further assume that $g(x) = x/a(x)$ is invertible on an interval $]0, \epsilon'']$ with $\epsilon''>0$.
\end{assumption}
For instance, ReLUs $\phi(x)=\max(x,0)$ inflict zero error on $\mathbb{R}$ (i.e., $a = \infty$) with $m_{+} = 1$, $m_{-} = 0$, and $d=0$.
Leaky ReLUs can be represented without error with $m_{+} = 1$, $m_{-} = \alpha$, and $d=0$ for an $\alpha > 0$.
$\phi(x)=\tanh(x)$ is approximately linear so that $|\tanh(x)-x| \leq x^3/3$ for $|x| < \pi/2$, which can be seen by Taylor expansion of $\tanh$. 
This implies that the choice $m_{+} = 1$, $m_{-} = 1$, and $d=0$ with $a = \min\{ (3 \epsilon')^{1/3}, \pi/2\}$ fulfills our assumption.
Sigmoids $\phi(x) = 1/(1+\exp(-x))$ can be analyzed in the same way with $m_{+} = m_{-} = 0.25$, $d=0.5$, and $a = \min\{ (48 \epsilon')^{1/3}, \pi\}$, since $\phi(x) = (\tanh(x/2)+1)/2$. 

All these activation functions can approximate the identity as $\left|x - r \left(\phi(x)-\phi(-x)\right)\right| \leq 2 r \epsilon'$, where we have defined $r := \frac{1}{m_{+} + m_{-}}$.
To abbreviate our notation later on, we also use $\mu_{\pm}(x) := m_{+}$ for $x > 0$, $\mu_{\pm}(x) := m_{-}$ for $x < 0$, and $\mu_{\pm}(0) := 0$ for $x=0$. 
Note that we always have $\mu_{\pm}(x) + \mu_{\pm}(-x) = m_{+} + m_{-} = 1/r$. 
Functions with $m_{+} = m_{-} = m$ and $d=0$ like \textsc{tanh} can also be approximated by $|x - \phi(x)/m| \leq \epsilon'/m$ and do not need separate approximations of the positive and the negative part.
The ability to easily approximate the identity is a valuable property to construct target networks that fit the depth of the source network.
We can always increase the depth of a target by concatenating identity approximating layers.
Note that we otherwise only need this assumption in the middle layers of the ($2L_t$)-construction and the first layer of the ($L_t+1$)-construction. 
Any other activation functions do not even need to fulfill this assumption.

\section{Existence Results}
The main idea in our construction of LTs rests on the linearity of convolutions.
Concretely, let us assume for a moment that the input tensors in Eq.~(\ref{eq:DNN}) of the LT are identical to a target input up to a scalar factor $\lambda_j$, i.e., $\bm{x}^{(l-1)}_{j'} = \lambda_{j'} \bm{x}^{(l-1)}_{t,j}$ for all $j' \in I_j$.
It follows that $\sum_{j'\in I_j} \bm{W}^{(l)}_{ij'} * \bm{x}^{(l-1)}_{j'} =  \sum_{j'\in I_j} \bm{W}^{(l)}_{ij'} * (\lambda_{j'} \bm{x}^{(l-1)}_{t,j}) = \left(\sum_{j' \in I_j} \bm{W}^{(l)}_{ij'} \lambda_{j'} \right)* \bm{x}^{(l-1)}_{t,j}$. 
We could therefore use $\left(\sum_{j' \in I_j} \bm{W}^{(l)}_{ij'} \lambda_{j'} \right)$ in the LT to approximate a target tensor $\bm{W}^{(l)}_{t,ij}$ by masking some of the components $w^{(l)}_{ij'q}$ according to subset sum approximation.
In fact, all entries for the indices $j'q$ can be used independently to approximate the corresponding target tensor entry with index $q$. Hence, our main task is to create multiple candidates $\lambda_{j'} \bm{x}^{(l-1)}_{t,j}$ for an input $\bm{x}^{(l-1)}_{t,j}$.
Our two different construction approaches differ in how they achieve this.

\subsection{Two Layers for One}
Informally, our first objective is to show that for any convolutional and/or residual target network $f_t$ with depth $L_t$, maximum channel size $c_t$, and kernel size $k_t$, there exists with probability $1-\delta$ a sub-network $f_{\epsilon}$ of a source network with depth $L_0=2L_t$ that approximates the target up to error $\epsilon > 0$ if the source network has maximum channel size $c_0 \geq C c_t s_0 \log(c_t k_t L_t /\min\{\epsilon, \delta\}$ for a constant $C$ that is independent of $\epsilon$, $\delta$, $c_t$, $k_t$, and the stride $s_0$ of filters in uneven layers of $f_0$.

Why do the source network and the LT have twice the depth?
Fig.~\ref{fig:constr2l} visualizes the answer. 
Every neuron in an even layer $l=2l'$ of the LT approximates the corresponding neuron in layer $l'$ of the target so that  $\bm{x}^{2l'}_{0,j} \approx \bm{x}^{l'}_{t,j}$ and $c_{0,2l'} = c_{t,l'}$ for every $l' \in [L_t]$.
What happens in the uneven layers $l=2l'+1$?
We create multiple versions of the input neurons $\bm{x}^{(2(l'-1))}_{0,j}$ to support our subset sum approximation problems and we achieve this with the help of univariate filters.
Note that the filters of $f_0$ do not need to be univariate themselves.
We can also prune them into this state by setting all filter entries except for one to zero.
Which filter entry we keep as nonzero depends on the stride.
For simplicity, let us assume here that we have univariate filters $w^{(2l'-1)}_{i'j1}$ with stride $s_0 = 1$ already.
How to transfer other cases to this setting is discussed in the appendix.

We have $w^{(2l'-1)}_{i'j1} = \lambda_{i'j}$ for every $i' \in I_j$ and $w^{(2l'-1)}_{i'j1} = 0$ otherwise.
This creates preactivations $\bm{h}^{(2l'-1)}_{i'}  = \sum_{j'} \bm{W}^{(2l'-1)}_{i'j'} * \bm{x}^{(2l'-2)}_{0,j'}  = \lambda_{i'j}  \bm{x}^{(2l'-2)}_{0,j}$, which is exactly what we were looking for.
Yet, we still have to send each entry through an activation function receiving $\phi(\lambda_{i'j}  \bm{x}^{(2l'-2)}_{0,j})$. 
To handle input entries with different signs, we have to create neurons with positive $(\lambda_{i'j} > 0)$ and negative ($\lambda_{i'j} < 0$).  
For $|\lambda_{i'j}|$ small enough, we can then approximate $\phi$ as in Assumption~\ref{def:act}, use this to construct the identity, and exchange the order of the summation and convolution because of the local linearity of $\phi$. 
We receive $\sum_{i' \in I_j} \bm{W}^{(2l')}_{0,ii'}  * \phi(\lambda_{i'j}  \bm{x}^{(2l'-2)}_{0,j})$ $\approx  \left(\sum_{i' \in I_j, \lambda_{i'j} > 0} \bm{W}^{(2l')}_{0,ii'} \lambda_{i'j} \right) \left(\mu_{\pm}(\bm{x}^{(2l'-2)}_{0,j}) \bm{x}^{(2l'-2)}_{0,j} \right) + \left(\sum_{i' \in I_j, \lambda_{i'j} < 0} \bm{W}^{(2l')}_{0,ii'} \lambda_{i'j} \right) *   \left(\mu_{\pm}(-\bm{x}^{(2l'-2)}_{0,j}) \bm{x}^{(2l'-2)}_{0,j} \right)$ $\approx \bm{W}^{(l')}_{t,ij} *\left[r \left(\mu_{\pm}(-\bm{x}^{(2l'-2)}_{0,j}) + \mu_{\pm}(-\bm{x}^{(2l'-2)}_{0,j}\right) \bm{x}^{(2l'-2)}_{0,j}\right]$ $\approx \bm{W}^{(l')}_{t,ij} * \bm{x}^{(l'-1)}_{t,j}$,
if we can approximate $r \sum_{i',  \lambda_{i'j} > 0} w^{(2l')}_{0,ii'q} \lambda_{i'j} \approx  w^{(l')}_{t,ijq}$ and $r \sum_{i',  \lambda_{i'j} < 0} w^{(2l')}_{0,ii'q} \lambda_{i'j} \approx  w^{(l')}_{t,ijq}$ for all filter entries $q$ by appropriate masking of the tensor elements $w^{(l')}_{t,ijq}$.
In addition, we have used that $r \left(\mu_{\pm}(-\bm{x}^{(2l'-2)}_{0,j}) + \mu_{\pm}(-\bm{x}^{(2l'-2)}_{0,j}\right) = 1$ by definition and that we have $\bm{x}^{(2l'-2)}_{0,j} \approx \bm{x}^{(l'-1)}_{t,j} $ by construction. 
Biases can be obtained similarly by approximating $\sum_{i',q} w^{(2l')}_{0,ii'} \mu_{\pm}(\bm{b}^{(2l'-1)}_{0,i'})\bm{b}^{(2l'-1)}_{0,i'} \approx \bm{b}^{(l')}_{t,i}$.

\begin{theorem}[LT existence ($2L_t$-construction)]\label{thm:LTexist2L}
Assume that ${\epsilon, \delta \in {(0,1)}}$, a convolutional target network (without skip connections) $f_t(x): \mathcal{D} \subset \mathbb{R}^{c_0 \times d_0} \rightarrow \mathbb{R}^{c_L \times d_{L_t}}$ with architecture $\bar{c}_t$ of depth $L_t$ with $N_{t,l}$ nonzero parameters in Layer $l$, and a source network $f_0$ with architecture $\bar{n}_0$ of depth $L_0 = 2 L_t$ are given.
Let $\phi$ be the activation function of $f_t$ with Lipschitz constant $T$ fulfilling Assumption~\ref{def:act} with $d=0$. 
Then, with probability at least $1-\delta$, $f_{0}$ contains a subnetwork $f_{\epsilon} \subset f_{0}$ so that each output component $i$ is approximated as $\max_{\bm{x}\in \mathcal{D}} \left|f_{t,iq}(\bm{x})- f_{\epsilon',iq}(\bm{x}) \right| \leq \epsilon$, %
if for all $l' \in [L_t]$ we have 
\begin{align*}
  c_{0,2l'+1} \geq  C c_{t,l} \log\left(\frac{N_t }{\min\{\epsilon/\prod^{L_t}_{s=l}(3TN_{t,s}), \delta \} }\right), %
\end{align*}
and $n_{0,2l'} \geq n_{t,l'}+1$, and if the parameters of $f_{0}$ are initialized according to Assumption~\ref{def:initconv} with $\sigma_{2l'+1} = r/\sigma_{2l'} $ and $\sigma_{2l'} = a(\epsilon'')/2$ and $\epsilon'' = g^{-1} \left(\frac{\epsilon'}{C  N_t \log\left(\frac{N_t }{\min\{\epsilon/\prod^{L_t}_{s=l}(3TN_{t,s}), \delta \} }\right) } \right)$ for $g(\epsilon'') = \epsilon''/(a(\epsilon''))$.
\end{theorem}
The proof is provided in the appendix.
In summary, two main insights enable the success of this construction and allow for generally positive and negative inputs in contrast to \citep{cnnexist}.
First, we convolute the input channels with an univariate filter in the first layer (and not in the second as \citep{cnnexist}).
Second, the insight that we can then prune the entries of each filter in the second layer independently, which follows from the linearity of convolutions, makes the construction parameter efficient and flexible.

We could derive a more advantageous scaling of the error if we would additionally assume that $\norm{W^{(l)}_t}_2 \leq 1$.  
Note also that the error does not depend on the input tensor dimension, i.e., the image size. 
This would change if we also incorporated a common flattening operation and fully-connected layers, which are often used to solve classification problems in the end. 
In this case, the image dimension would determine the number of input features to the fully-connected layers.
As these are handled in a different work \citep{depthexist}, we skip a deeper discussion. 
We would only need to adapt the initial $\epsilon$ to combine them. 

We only state the theorem for activation functions $\phi(0) = d = 0$ here. 
For $d \neq 0$, we can derive similar results if we change the initialization scheme to 'looks-linear' initialization \cite{dyniso} to control the error when we approximate $\phi$, see also \cite{depthexist}.

\subsection{($L_0=L_t+1$)-Construction}
The $(2L_t)$-construction uses an additional layer to create multiple versions of the input for each target layer. 
As the neuron states are sent through the non-linear activation function, we usually need to create two neurons, the analog to the positive and the analog to the negative part of each input channel.
Yet, this extra effort would not be necessary if we had immediately multiple versions of each input channel available.
As we cannot change the given input by the data, to create multiple versions initially, we have to employ the two-for-one layer construction to approximate the first target layer.
If we directly construct the $c_{t,1}$ output channels multiple times, we can however drop the next intermediary layer completely and repeat constructing the channels of the next layer so many times that they serve the solution of subset sum approximation problems in the following layers. 
Fig.~\ref{fig:lp1l} explains the main idea and the next theorem states formally the existence result for this construction.
\begin{figure}
\includegraphics[width=0.49\textwidth]{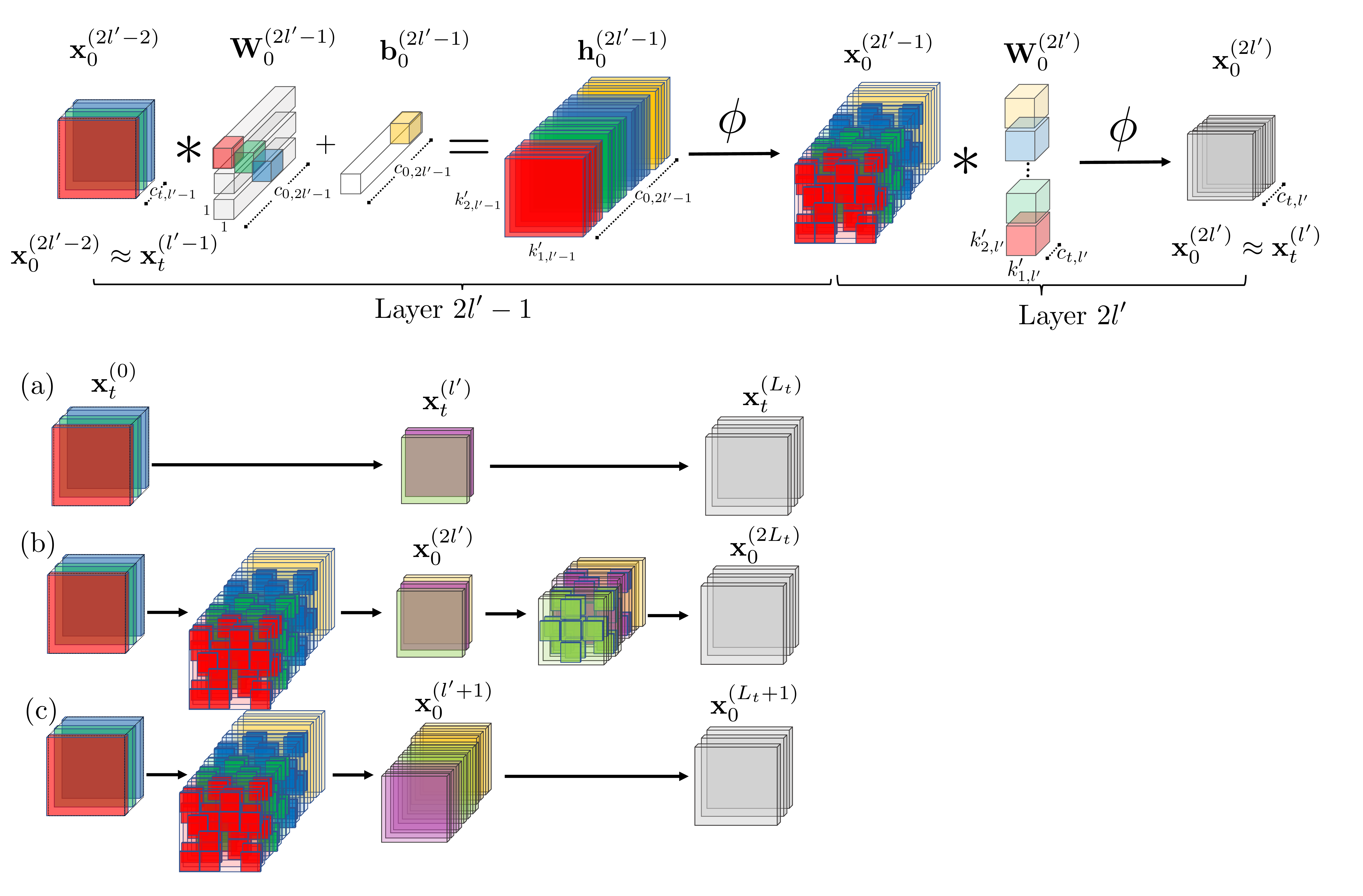}
\caption{Comparison of LT construction approaches. (a) Target network $f_t$. (b) $(2L_t)$-construction. (c) $(L_t+1)$ construction.}\label{fig:lp1l}
\end{figure}
\begin{theorem}[LT existence ($2L_t$-construction)]\label{thm:LTexistfull}
Assume that ${\epsilon, \delta \in {(0,1)}}$, a convolutional target network (possibly with skip connections) $f_t(x): \mathcal{D} \subset \mathbb{R}^{c_0 \times d_0} \rightarrow \mathbb{R}^{c_L \times d_{L_t}}$ with architecture $\bar{c}_t$ of depth $L_t$, and a source network $f_0$ with architecture $\bar{n}_0$ of depth $L_0 = L_t+1$ are given.
Let $\phi$ be the activation function of $f_t$ and $f_0$ with Lipschitz constant $T$. 
Furthermore, let $\phi_0$ be the activation function of $f_0$ in the first layer fulfilling Assumption~\ref{def:act} with $d=0$. 
Define the number $N_{l}$ of effective nonzero parameters in Layer $l$ as $N_{l} = N_{w,l} + N_{m,l}$.
Then, with probability at least $1-\delta$, $f_{0}$ contains a subnetwork $f_{\epsilon} \subset f_{0}$ so that each output component is approximated as $\max_{\bm{x}\in \mathcal{D}} \left|f_{t,iq}(\bm{x})- f_{\epsilon',iq}(\bm{x}) \right| \leq \epsilon$ %
if for all $l \in [L_t]$ we have 
\begin{align*}
  c_{0,l+1} \geq  C c_{t,l} \log\left(\frac{1 }{\min\{\epsilon_l, \rho \delta/N_l \} }\right), %
\end{align*}
and $n_{0, 1} \geq C c_{t,0} \log\left(\frac{1 }{\min\{\epsilon_1, \delta \rho \} }\right)$ with $\rho = C N^{1+\gamma}_l \log(1/\min\{ \min_l \epsilon_{l}, \delta \})$ for any $\gamma > 0$.
$\epsilon_l = \frac{\epsilon}{2 T N_{w,l} \prod^L_{s=l+1} (2 (T N_{w,s} + N_{m,s}))}$.
Additionally, we require that the parameters of $f_{0}$ are initialized according to Assumption~\ref{def:initconv} with $\sigma_l = 1$ for $l > 2$, $\sigma_{1} = r/\sigma_{2} $ and $\sigma_{2} = a(\epsilon'')/2$ and suitably chosen $\epsilon''$. 
\end{theorem}
The proof is given in the appendix.
The main challenge in the derivation is to identify the size of the required subset sum blocks, as it depends on the number of total subset sum problems that need to be solved.
This number in turn depends on the subset sum block sizes.
Both need to be balanced as stated in the theorem.
We observe that the block size is potentially larger in the ($L_t+1$)-construction than in the ($2L_t$)-construction if more subset sum problems need to be solved to create multiple versions of the previous channel directly.
But this factor enters only the logarithm and is thus small.
The fact that the ($2L_t$)-construction requires often double the amount of channels to regard the analogs of the positive and negative part separately, is usually a stronger requirement.
For a very high number of target parameters, the required $L_t+1$ blocks might still be larger.
Furthermore, the LT in the $L_t+1$-construction consists of many more parameters, yet, less neurons, which is often dominating the computational needs on GPUs.
We discuss these differences in detail in the experiments section.
Regardless of the advantages and disadvantages of each construction, the main purpose of this theorem is to show that LTs exist that can leverage most of the source network's depth.
Which construction could be found by different pruning algorithms and which one would be better is a different question.
\section{Experiments}
Our theoretical insights suggest that source networks do not need to be much larger than the networks that we want to approximate by a lottery ticket - at least with regard to how our width requirement scales with the relevant parameters.
Three possible challenges could still arise in practice.
First, the size of the constant in our width requirement might be impractically large.
Second, deep target networks consisting of many parameters might be very fragile to small errors in their parameters so that $\epsilon_l$ is so small that even $\log(1/\epsilon_l)$ in our width requirement is practically too large.
Similarly, it could be the case that we would need to solve so many subset sum problems, in particular in the one-layer-for-one construction, that $\delta/N$ would be too small.
The following experiments rule out these three concerns and show that constructing lottery tickets by solving subset sum approximations is practically feasible.

\begin{figure*}
\includegraphics[width=\textwidth]{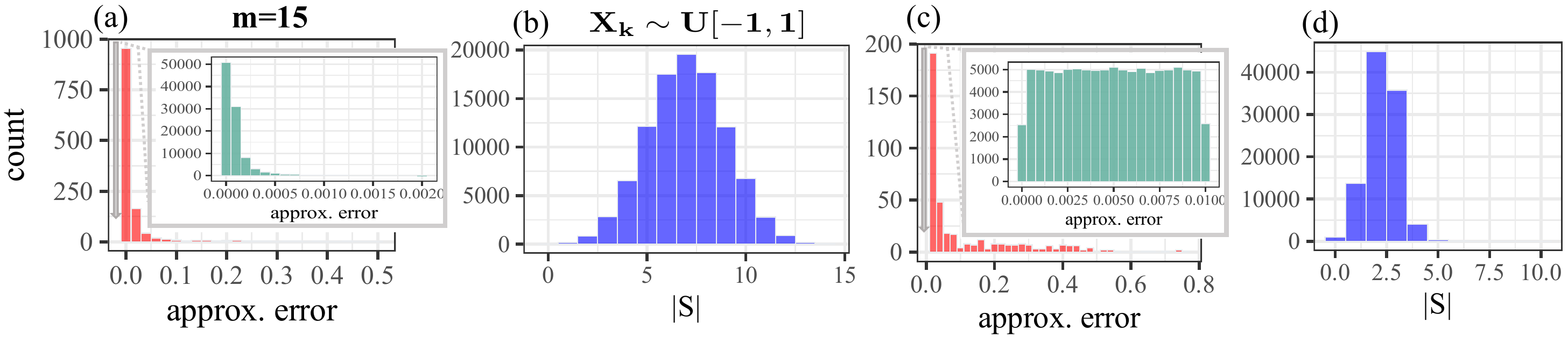}
\caption{Subset sum approximation statistics. (a) \& (c): Error of solving $10^5$ independent subset sum problems. Each problem selects $|S|$ elements out of $m=15$ independent random variables $X_k \sim U[-1,1]$ to approximate a randomly drawn target $z\sim U[-1,1]$. The green histogram in the right corner focuses on the smallest errors. (b) \& (d): Size of approximating subset. (a) \& (b): Best subset selection based on exhaustive search. (c) \& (d): Smallest subset that achieves an error of maximally $\epsilon = 0.01$.}
  \label{fig:subsetsumstats}
\end{figure*}

In fact, our proofs define implicitly an algorithm that approximates a target network by pruning a source network consisting of $m (c_{t,l}+1) + c_0$ channels in each convolutional layer with $l > 1$, $2 m_{2} (c_{t,0}+1) + c_0$ channels in the first layer, and $c_{t,L}$ channels in the output layer.
$m$ here refers to the size of subset sum blocks that we choose. 
$c_0$ gives us the option to fail sometimes in solving a subset sum problem and continue with pruning another neuron instead.
It is negligible in our experiments. 

Thus, our first question has to be: How large should we choose $m$ (and $m_2$)?
Even though solving subset sum problems is in general NP-hard, for small enough $m$ we can still solve them optimally by exhaustively evaluating all subsets. 
Fig.~\ref{fig:subsetsumstats}~(a) presents statistics on these optimal solutions for $m=15$, which is generally considered to be sufficiently large in the literature.
On average, the error within the 95\% standard confidence interval is $(6 \pm 0.8) 10^{-4}$, but in most cases it is much smaller, as the left skewed distribution indicates.
It exceeds $0.01$ only in $0.3\%$ of the cases.

In addition to the approximation error, we are also interested in the number of parameters that enter our LT and thus the size of the selected subset $|S|$. 
If our target layer consists of $N_l$ nonzero parameters, we expect that our LT  consists on average of $m \mathds{E}(|S|) N_l$ parameters. %
Fig.~\ref{fig:subsetsumstats}~(b) shows the distribution of $|S|$, which has $\mathds{E}(|S|) \approx 7$. 
For increasing $m$, the average error decreases but the average subset size would increase, which is not ideal for constructing sparse tickets.
However, LTs do not require the optimal subset.
A subset that reaches a small enough error would be sufficient.
Fig.~\ref{fig:subsetsumstats}~(c-d) show therefore statistics for a subset selection process that searches through all subsets but stops when it finds a subset whose approximation error does not exceed $0.01$. 
In this case, we only need $\mathds{E}(|S|) \approx 2.3$.
This observation is also relevant for the LT's sparsity assessment.
Commonly, sparsity is reported relative to a dense network.
Thus, if the target layer has sparsity $\rho_t = N_l/(c_{t,l} c_{t,l-1} K)$, the LT has roughly $\rho_{\epsilon} = m \mathds{E}(|S|) N_l/([m (c_{t,l}+1) + c_0][m (c_{t,l-1}+1) + c_0]  K) \approx  \rho_t \mathds{E}(|S|)/m$.
Accordingly, the choice $m=15$ and allowed error $\epsilon_l = 0.01$ let's us construct LTs that have 15\% of the target's sparsity.
We could always reduce this number further by equipping the source network with a higher number of channels that we can prune away.
The real quantity of interest is therefore the total number of required nonzero parameters, which we report in the following experiments alongside the neural network performance. 

We could repeat a similar analysis for the two-layers-for-one construction.
In this case, we would have to solve subset sum problems with random variables that are distributed as the product of two uniform random variables $X_k = U[-1, 1] U[-1, 1]$.
Statistics regarding this case are presented in the appendix in Fig.~\ref{fig:subsetsumstats2}.
The overall distributions look similar. 
A base set size of $m=15$ random variables is also sufficient but results in an average error of $(7.6 \pm 0.3) 10^{-3}$, which is more than a magnitude higher than in the previous case.
Furthermore, we fail to achieve an error smaller than $0.01$ in $3.5 \%$ of the cases and rely on subset sizes of $\mathds{E}(|S|) \approx 2.5$. 
This only affects the first layer of our experiments though.

\begin{table}[t]
\caption{Test accuracy in \%  of pruned LTs and their $L=3$ target on MNIST for $(L+1)$ construction. Averages and $0.95$ standard confidence intervals are reported for $10$ repetitions.} %
\label{table_simple}
\vskip 0.15in
\begin{center}
\begin{small}
\begin{sc}
\begin{tabular}{lccr}
\toprule
 & Target  & LT  \\
\midrule
ReLU & 98.8 & 98.72 $\pm$ 0.04\\ 
Leaky ReLU  & 98.5 & 98.5 $\pm$ 0.03 \\ 
Tanh & 98.14 & 98.09 $\pm$ 0.09 \\ 
Sigmoid & 98.52 & 98.5 $\pm$ 0.004 \\ 
\bottomrule
\end{tabular}
\end{sc}
\end{small}
\end{center}
\vskip -0.1in
\end{table}

To demonstrate that an error of $\epsilon_l=0.01$ per parameter is indeed acceptable to obtain LTs most of the time, we employ the described pruning strategy that solves subset sum approximation problems with small subsets that try to achieve an error of maximally $\epsilon_l=0.01$ with respect to each parameter of a given target network.
We identify two different types of target networks with the established Synflow algorithm and its open source code \citep{synflow}, which we apply with exponential annealing of the target sparsity with 12 steps on a machine with Intel(R) Core(TM) i9-10850K CPU @ 3.60GHz processor and GPU NVIDIA GeForce RTX 3080 Ti. 
Between pruning steps, we train the pruned network for $50$ epochs. 

The first target network type is a small scale example with 3 layers and is pruned and trained on MNIST \citep{mnist}.
Its first two layers are convolutional with 32 channels and 3x3 filters before pruning. 
The last layer is a fully-connected classification layer with softmax activation functions. 
We obtain separate targets for four commonly used activation functions in the first two layers: \textsc{ReLU}, \textsc{leaky ReLU}, \textsc{sigmoid}, and \textsc{tanh} and explicitly approximate the convolutional layers with an $L+1$ construction by solving the associated subset sum problems. 
The performance of the resulting neural networks is reported in Table~\ref{table_simple} and the number of the pruned nonzero parameters in Table~\ref{table2_simple}.
We observe a very similar performance of the LTs in comparison with the target networks.
While we report averages over $50$ independent runs, note that we can always find a couple of lucky solutions that outperform the target network on the test set. 
If we would use the optimal subset instead of a small one, we would see no significant difference between the target and the LT.
Yet, the resulting LT would also consist of more parameters.
The delicate trade-off between LT size and potential accuracy has to be solved in practice.

\begin{table}[t]
\caption{Number of prunable neural network parameters for experiments reported in Table~\ref{table_simple} regarding the MNIST example.} %
\label{table2_simple}
\vskip 0.15in
\begin{center}
\begin{small}
\begin{sc}
\begin{tabular}{lccr}
\toprule
 & Target  & LT  \\
\midrule
ReLU & 945 & 4386 $\pm$ 34 \\ 
Leaky ReLU & 940 & 4545 $\pm$ 44\\ 
Tanh &  950 & 3269 $\pm$ 19 \\ 
Sigmoid & 953 & 4198 $\pm$ 59 \\ 
\bottomrule
\end{tabular}
\end{sc}
\end{small}
\end{center}
\vskip -0.1in
\end{table}

\begin{table}[t]
\caption{Test accuracy in \%  of pruned LTs and their target for ResNet-22 on CIFAR10 for $(L+1)$ construction. Averages and $0.95$ standard confidence intervals are reported for $50$ repetitions.} %
\label{table}
\vskip 0.15in
\begin{center}
\begin{small}
\begin{sc}
\begin{tabular}{lcccr}
\toprule
 & Target  & LT & LT (best 50\%) \\
\midrule
ReLU & 80.68 & 80.32 $\pm$ 0.13 & 80.69 $\pm$ 0.09 \\ 
Leaky ReLU  & 80.68 & 80.11 $\pm$ 0.18 & 80.6 $\pm$ 0.1 \\ 
Tanh & 80.86  &  80.48 $\pm$ 0.13 & 80.8 $\pm$ 0.1 \\ 
Sigmoid & 72.66 & 69 $\pm$ 2 & 73.2 $\pm$ 0.9\\ 
\bottomrule
\end{tabular}
\end{sc}
\end{small}
\end{center}
\vskip -0.1in
\end{table}

The second target that we consider has a more realistic structure that is much deeper and includes residual blocks.
Similar to before, we prune and train ResNet-22 on CIFAR-10 \citep{cifar10}. 
To save computational time, we draw for each parameter and error and associated set size from the empirical distribution that we derived by solving $10^5$ independent subset sum problems as shown in Figure~\ref{fig:subsetsumstats}~(c-d) (and the appendix for the first layer).  
The results are presented in Tables~\ref{table}\&\ref{table2}.

\begin{table}[t]
\caption{Number of prunable neural network parameters for experiments reported in Table~\ref{table} regarding the ResNet-22 example.} %
\label{table2}
\vskip 0.15in
\begin{center}
\begin{small}
\begin{sc}
\begin{tabular}{lccr}
\toprule
 & Target  & LT  \\
\midrule
ReLU & 31069 & 1058140 $\pm$ 559 \\ 
Leaky ReLU  & 31113 & 1059537 $\pm$ 567 \\ 
Tanh & 64491 & 2193086 $\pm$ 723\\ 
Sigmoid & 116048 & 3934621 $\pm$ 1248 \\ 
\bottomrule
\end{tabular}
\end{sc}
\end{small}
\end{center}
\vskip -0.1in
\end{table}

\section{Conclusions}
We have proven that pruning randomly initialized convolutional neural networks, including residual blocks and skip connections, can be a viable strategy to identify smaller scale networks.
These lottery tickets (LTs) can be found if they approximate target networks whose width is smaller by a logarithmic factor than the original random source network. 
In practice, a factor of $1/15$ is sufficient, as we have verified in experiments.
Our proofs are the first to cover residual and skip connections and other activation functions than ReLUs for convolutional layers.
We have furthermore presented a novel LT construction for convolutional layers that is not restricted to positive inputs and discussed two versions.
The first one assumes that the depth of the target network is $L_t \leq L_0/2$.
This construction is more parameter efficient but requires a relatively high depth and usually a higher number of neurons than the second version.
In the second version, the target depth can only be slightly smaller than the depth of the large random network $L_t \leq L_0-1$. 
While the resulting LT consists of many more parameters relative to the target representation, the target representation itself can be much sparser, as it is allowed to utilize most of the available depth $L_0$.
Furthermore, this construction indicates that not only extremely deep networks contain LTs. 
This is an important finding, as we can thus focus our pruning efforts on neural networks of similar depth as contemporary architectures, which are feasible to train. %

\bibliography{ref}

\begin{thebibliography}{51}
\providecommand{\natexlab}[1]{#1}
\providecommand{\url}[1]{\texttt{#1}}
\expandafter\ifx\csname urlstyle\endcsname\relax
  \providecommand{\doi}[1]{doi: #1}\else
  \providecommand{\doi}{doi: \begingroup \urlstyle{rm}\Url}\fi

\bibitem[Burkholz(2022)]{depthexist}
Burkholz, R.
\newblock Most activation functions can win the lottery without excessive
  depth.
\newblock In \emph{arXiv}, 2022.

\bibitem[Burkholz \& Dubatovka()Burkholz and Dubatovka]{dyniso}
Burkholz, R. and Dubatovka, A.
\newblock Initialization of {ReLUs} for dynamical isometry.
\newblock In \emph{Advances in Neural Information Processing Systems}.

\bibitem[Burkholz et~al.(2022)Burkholz, Laha, Mukherjee, and Gotovos]{uniExist}
Burkholz, R., Laha, N., Mukherjee, R., and Gotovos, A.
\newblock On the existence of universal lottery tickets.
\newblock In \emph{International Conference on Learning Representations}, 2022.

\bibitem[Chen et~al.(2021{\natexlab{a}})Chen, Cheng, Gan, Liu, and
  Wang]{chen2021dataefficient}
Chen, T., Cheng, Y., Gan, Z., Liu, J., and Wang, Z.
\newblock Data-efficient {GAN} training beyond (just) augmentations: A lottery
  ticket perspective.
\newblock In \emph{Advances in Neural Information Processing Systems},
  2021{\natexlab{a}}.

\bibitem[Chen et~al.(2021{\natexlab{b}})Chen, Sui, Chen, Zhang, and
  Wang]{graphLTH}
Chen, T., Sui, Y., Chen, X., Zhang, A., and Wang, Z.
\newblock A unified lottery ticket hypothesis for graph neural networks.
\newblock In \emph{International Conference on Machine Learning},
  2021{\natexlab{b}}.

\bibitem[Chen et~al.(2021{\natexlab{c}})Chen, Cheng, Wang, Gan, Liu, and
  Wang]{elasticLTH}
Chen, X., Cheng, Y., Wang, S., Gan, Z., Liu, J., and Wang, Z.
\newblock The elastic lottery ticket hypothesis.
\newblock In \emph{Advances in Neural Information Processing Systems},
  2021{\natexlab{c}}.

\bibitem[da~Cunha et~al.(2022)da~Cunha, Natale, and Viennot]{cnnexist}
da~Cunha, A., Natale, E., and Viennot, L.
\newblock Proving the lottery ticket hypothesis for convolutional neural
  networks.
\newblock In \emph{International Conference on Learning Representations}, 2022.

\bibitem[Deng(2012)]{mnist}
Deng, L.
\newblock The mnist database of handwritten digit images for machine learning
  research.
\newblock \emph{IEEE Signal Processing Magazine}, 29\penalty0 (6):\penalty0
  141--142, 2012.

\bibitem[Diffenderfer \& Kailkhura(2021)Diffenderfer and Kailkhura]{multiprize}
Diffenderfer, J. and Kailkhura, B.
\newblock Multi-prize lottery ticket hypothesis: Finding accurate binary neural
  networks by pruning a randomly weighted network.
\newblock In \emph{International Conference on Learning Representations}, 2021.

\bibitem[Dong et~al.(2017)Dong, Chen, and Pan]{dong2017surgeon}
Dong, X., Chen, S., and Pan, S.~J.
\newblock Learning to prune deep neural networks via layer-wise optimal brain
  surgeon.
\newblock In \emph{Advances in Neural Information Processing Systems}, 2017.

\bibitem[Evci et~al.(2020)Evci, Gale, Menick, Castro, and
  Elsen]{evci2019rigging}
Evci, U., Gale, T., Menick, J., Castro, P.~S., and Elsen, E.
\newblock Rigging the lottery: Making all tickets winners.
\newblock In \emph{International Conference on Machine Learning}, 2020.

\bibitem[Fischer \& Burkholz(2021)Fischer and Burkholz]{nonzerobiases}
Fischer, J. and Burkholz, R.
\newblock Towards strong pruning for lottery tickets with non-zero biases,
  2021.

\bibitem[Fischer \& Burkholz(2022)Fischer and Burkholz]{plant}
Fischer, J. and Burkholz, R.
\newblock Plant 'n' seek: Can you find the winning ticket?
\newblock In \emph{International Conference on Learning Representations}, 2022.

\bibitem[Frankle \& Carbin(2019)Frankle and Carbin]{frankle2019lottery}
Frankle, J. and Carbin, M.
\newblock The lottery ticket hypothesis: Finding sparse, trainable neural
  networks.
\newblock In \emph{International Conference on Learning Representations}, 2019.

\bibitem[Frankle et~al.(2020)Frankle, Dziugaite, Roy, and Carbin]{rewind}
Frankle, J., Dziugaite, G.~K., Roy, D., and Carbin, M.
\newblock Linear mode connectivity and the lottery ticket hypothesis.
\newblock In \emph{International Conference on Machine Learning}, 2020.

\bibitem[Frankle et~al.(2021)Frankle, Dziugaite, Roy, and
  Carbin]{frankle2021review}
Frankle, J., Dziugaite, G.~K., Roy, D., and Carbin, M.
\newblock Pruning neural networks at initialization: Why are we missing the
  mark?
\newblock In \emph{International Conference on Learning Representations}, 2021.

\bibitem[Han et~al.(2015)Han, Pool, Tran, and Dally]{IMPfirst}
Han, S., Pool, J., Tran, J., and Dally, W.
\newblock Learning both weights and connections for efficient neural network.
\newblock In \emph{Advances in Neural Information Processing Systems}, 2015.

\bibitem[Hassibi \& Stork(1992)Hassibi and Stork]{hassibi1992second}
Hassibi, B. and Stork, D.~G.
\newblock Second order derivatives for network pruning: Optimal brain surgeon.
\newblock In \emph{International Conference on Neural Information Processing
  Systems}, 1992.

\bibitem[Krizhevsky(2009)]{cifar10}
Krizhevsky, A.
\newblock Learning multiple layers of features from tiny images.
\newblock 2009.

\bibitem[LeCun et~al.(1990{\natexlab{a}})LeCun, Denker, and Solla]{braindamage}
LeCun, Y., Denker, J., and Solla, S.
\newblock Optimal brain damage.
\newblock In \emph{Advances in Neural Information Processing Systems},
  1990{\natexlab{a}}.

\bibitem[LeCun et~al.(1990{\natexlab{b}})LeCun, Denker, and
  Solla]{lecun1990optimal}
LeCun, Y., Denker, J.~S., and Solla, S.~A.
\newblock Optimal brain damage.
\newblock In \emph{Advances in neural information processing systems}, pp.\
  598--605, 1990{\natexlab{b}}.

\bibitem[Lee et~al.(2019)Lee, Ajanthan, and Torr]{snip}
Lee, N., Ajanthan, T., and Torr, P. H.~S.
\newblock Snip: single-shot network pruning based on connection sensitivity.
\newblock In \emph{International Conference on Learning Representations}, 2019.

\bibitem[Lee et~al.(2020)Lee, Ajanthan, Gould, and Torr]{orthoRepair}
Lee, N., Ajanthan, T., Gould, S., and Torr, P. H.~S.
\newblock A signal propagation perspective for pruning neural networks at
  initialization.
\newblock In \emph{International Conference on Learning Representations}, 2020.

\bibitem[Li et~al.(2017)Li, Kadav, Durdanovic, Samet, and
  Graf]{li2017pruneconv}
Li, H., Kadav, A., Durdanovic, I., Samet, H., and Graf, H.~P.
\newblock Pruning filters for efficient convnets.
\newblock In \emph{International Conference on Learning Representations}, 2017.

\bibitem[Liu et~al.()Liu, Yuan, Che, Shen, Ma, Jin, Ren, Tang, Liu, and
  Wang]{liu2021:finetune}
Liu, N., Yuan, G., Che, Z., Shen, X., Ma, X., Jin, Q., Ren, J., Tang, J., Liu,
  S., and Wang, Y.
\newblock Lottery ticket preserves weight correlation: Is it desirable or not?
\newblock In \emph{International Conference on Machine Learning}.

\bibitem[Liu et~al.(2021{\natexlab{a}})Liu, Yuan, Che, Shen, Ma, Jin, Ren,
  Tang, Liu, and Wang]{weightcor}
Liu, N., Yuan, G., Che, Z., Shen, X., Ma, X., Jin, Q., Ren, J., Tang, J., Liu,
  S., and Wang, Y.
\newblock Lottery ticket preserves weight correlation: Is it desirable or not?
\newblock In \emph{International Conference on Machine Learning},
  2021{\natexlab{a}}.

\bibitem[Liu et~al.(2021{\natexlab{b}})Liu, Chen, Chen, Atashgahi, Yin, Kou,
  Shen, Pechenizkiy, Wang, and Mocanu]{plastic}
Liu, S., Chen, T., Chen, X., Atashgahi, Z., Yin, L., Kou, H., Shen, L.,
  Pechenizkiy, M., Wang, Z., and Mocanu, D.~C.
\newblock Sparse training via boosting pruning plasticity with
  neuroregeneration.
\newblock In \emph{Advances in Neural Information Processing Systems},
  2021{\natexlab{b}}.

\bibitem[Lueker(1998)]{subsetsum}
Lueker, G.~S.
\newblock Exponentially small bounds on the expected optimum of the partition
  and subset sum problems.
\newblock \emph{Random Structures \& Algorithms}, 12\penalty0 (1):\penalty0
  51--62, 1998.

\bibitem[Ma et~al.(2021)Ma, Yuan, Shen, Chen, Chen, Chen, Liu, Qin, Liu, Wang,
  and Wang]{sanity2}
Ma, X., Yuan, G., Shen, X., Chen, T., Chen, X., Chen, X., Liu, N., Qin, M.,
  Liu, S., Wang, Z., and Wang, Y.
\newblock Sanity checks for lottery tickets: Does your winning ticket really
  win the jackpot?
\newblock In \emph{Advances in Neural Information Processing Systems}, 2021.

\bibitem[Malach et~al.(2020)Malach, Yehudai, Shalev-Schwartz, and
  Shamir]{malach2020proving}
Malach, E., Yehudai, G., Shalev-Schwartz, S., and Shamir, O.
\newblock Proving the lottery ticket hypothesis: Pruning is all you need.
\newblock In \emph{International Conference on Machine Learning}, 2020.

\bibitem[Mhaskar et~al.(2017)Mhaskar, Liao, and Poggio]{deepOverShallow}
Mhaskar, H., Liao, Q., and Poggio, T.
\newblock When and why are deep networks better than shallow ones?
\newblock In \emph{AAAI Conference on Artificial Intelligence}, pp.\
  2343–2349, 2017.

\bibitem[Molchanov et~al.(2017)Molchanov, Tyree, Karras, Aila, and
  Kautz]{molchanov2017pruneinf}
Molchanov, P., Tyree, S., Karras, T., Aila, T., and Kautz, J.
\newblock Pruning convolutional neural networks for resource efficient
  inference.
\newblock In \emph{International Conference on Learning Representations}, 2017.

\bibitem[Mozer \& Smolensky(1989)Mozer and Smolensky]{trimming}
Mozer, M.~C. and Smolensky, P.
\newblock Skeletonization: A technique for trimming the fat from a network via
  relevance assessment.
\newblock In \emph{Advances in Neural Information Processing Systems}, 1989.

\bibitem[Orseau et~al.(2020)Orseau, Hutter, and
  Rivasplata]{orseau2020logarithmic}
Orseau, L., Hutter, M., and Rivasplata, O.
\newblock Logarithmic pruning is all you need.
\newblock \emph{Advances in Neural Information Processing Systems}, 33, 2020.

\bibitem[Pensia et~al.(2020)Pensia, Rajput, Nagle, Vishwakarma, and
  Papailiopoulos]{pensia2020optimal}
Pensia, A., Rajput, S., Nagle, A., Vishwakarma, H., and Papailiopoulos, D.
\newblock Optimal lottery tickets via subset sum: Logarithmic
  over-parameterization is sufficient.
\newblock In \emph{Advances in Neural Information Processing Systems},
  volume~33, pp.\  2599--2610, 2020.

\bibitem[Ramanujan et~al.(2020)Ramanujan, Wortsman, Kembhavi, Farhadi, and
  Rastegari]{ramanujan2019whats}
Ramanujan, V., Wortsman, M., Kembhavi, A., Farhadi, A., and Rastegari, M.
\newblock What's hidden in a randomly weighted neural network?
\newblock In \emph{Computer Vision and Pattern Recognition}, pp.\
  11893--11902, 2020.

\bibitem[Renda et~al.(2020)Renda, Frankle, and Carbin]{rewindVsFinetune}
Renda, A., Frankle, J., and Carbin, M.
\newblock Comparing rewinding and fine-tuning in neural network pruning.
\newblock In \emph{International Conference on Learning Representations}, 2020.

\bibitem[Savarese et~al.(2020{\natexlab{a}})Savarese, Silva, and Maire]{LTreg}
Savarese, P., Silva, H., and Maire, M.
\newblock Winning the lottery with continuous sparsification.
\newblock In \emph{Advances in Neural Information Processing Systems},
  2020{\natexlab{a}}.

\bibitem[Savarese et~al.(2020{\natexlab{b}})Savarese, Silva, and
  Maire]{sigmoidl0}
Savarese, P., Silva, H., and Maire, M.
\newblock Winning the lottery with continuous sparsification.
\newblock In \emph{Advances in Neural Information Processing Systems},
  2020{\natexlab{b}}.

\bibitem[Srinivas \& Babu(2016)Srinivas and Babu]{srinivas2016gendropout}
Srinivas, S. and Babu, R.~V.
\newblock Generalized dropout.
\newblock \emph{CoRR}, abs/1611.06791, 2016.

\bibitem[Su et~al.(2020)Su, Chen, Cai, Wu, Gao, Wang, and Lee]{sanity}
Su, J., Chen, Y., Cai, T., Wu, T., Gao, R., Wang, L., and Lee, J.~D.
\newblock Sanity-checking pruning methods: Random tickets can win the jackpot.
\newblock In \emph{Advances in Neural Information Processing Systems}, 2020.

\bibitem[Tanaka et~al.(2020)Tanaka, Kunin, Yamins, and Ganguli]{synflow}
Tanaka, H., Kunin, D., Yamins, D.~L., and Ganguli, S.
\newblock Pruning neural networks without any data by iteratively conserving
  synaptic flow.
\newblock In \emph{Advances in Neural Information Processing Systems}, 2020.

\bibitem[Verdenius et~al.(2020)Verdenius, Stol, and Forré]{snipit}
Verdenius, S., Stol, M., and Forré, P.
\newblock Pruning via iterative ranking of sensitivity statistics, 2020.

\bibitem[Wang et~al.(2020)Wang, Zhang, and Grosse]{grasp}
Wang, C., Zhang, G., and Grosse, R.~B.
\newblock Picking winning tickets before training by preserving gradient flow.
\newblock In \emph{International Conference on Learning Representations}, 2020.

\bibitem[Weigend et~al.(1991)Weigend, Rumelhart, and Huberman]{weightelim}
Weigend, A., Rumelhart, D., and Huberman, B.
\newblock Generalization by weight-elimination with application to forecasting.
\newblock In \emph{Advances in Neural Information Processing Systems}, 1991.

\bibitem[Yarotsky(2018)]{approx}
Yarotsky, D.
\newblock Optimal approximation of continuous functions by very deep relu
  networks.
\newblock In \emph{Conference On Learning Theory}, pp.\  639--649, 2018.

\bibitem[You et~al.(2020)You, Li, Xu, Fu, Wang, Chen, Baraniuk, Wang, and
  Lin]{earlybird}
You, H., Li, C., Xu, P., Fu, Y., Wang, Y., Chen, X., Baraniuk, R.~G., Wang, Z.,
  and Lin, Y.
\newblock Drawing early-bird tickets: Toward more efficient training of deep
  networks.
\newblock In \emph{International Conference on Learning Representations}, 2020.

\bibitem[Zhang et~al.(2021{\natexlab{a}})Zhang, Wang, Liu, Chen, and
  Xiong]{LTgeneralization}
Zhang, S., Wang, M., Liu, S., Chen, P.-Y., and Xiong, J.
\newblock Why lottery ticket wins? a theoretical perspective of sample
  complexity on sparse neural networks.
\newblock In \emph{Advances in Neural Information Processing Systems},
  2021{\natexlab{a}}.

\bibitem[Zhang et~al.(2021{\natexlab{b}})Zhang, Chen, Chen, and
  Wang]{lessdatamore}
Zhang, Z., Chen, X., Chen, T., and Wang, Z.
\newblock Efficient lottery ticket finding: Less data is more.
\newblock In \emph{International Conference on Machine Learning},
  2021{\natexlab{b}}.

\bibitem[Zhang et~al.(2021{\natexlab{c}})Zhang, Jin, Zhang, Zhou, Zhao, Ren,
  Liu, Wu, Jin, and Dou]{validateManifold}
Zhang, Z., Jin, J., Zhang, Z., Zhou, Y., Zhao, X., Ren, J., Liu, J., Wu, L.,
  Jin, R., and Dou, D.
\newblock Validating the lottery ticket hypothesis with inertial manifold
  theory.
\newblock In \emph{Advances in Neural Information Processing Systems},
  2021{\natexlab{c}}.

\bibitem[Zhou et~al.(2019)Zhou, Lan, Liu, and Yosinski]{zhou2019deconstructing}
Zhou, H., Lan, J., Liu, R., and Yosinski, J.
\newblock Deconstructing lottery tickets: Zeros, signs, and the supermask.
\newblock In \emph{Advances in Neural Information Processing Systems}, pp.\
  3597--3607, 2019.

\end{thebibliography}
\bibliographystyle{icml}

\newpage
\appendix
\onecolumn
\section{Subset Sum Approximation}
We generally have multiple random neurons and parameters available to approximate a target parameter $z$ by $\widehat{z}$ up to error $\epsilon$ so that $|z-\widehat{z}| \leq \epsilon$. 
Let us denote these random parameters in the source network as $X_i$.
If these contain a uniform distribution, as defined below, we can utilize a subset of them for approximating $z$.
\begin{definition}\label{def:containUniform}
A random variable $X$ contains a uniform distribution if there exist constants $\alpha \in (0,1]$, $c, h > 0$  and a random variable $G_1$ so that $X$ is distributed as $X \sim \alpha U[c-h,c+h] + (1-\alpha) G_1$.
\end{definition}
\citep{uniExist} extended results by \citep{subsetsum} to solve subset sum approximation problems if the random variables are not necessarily identically distributed. In addition, they also cover the case $|z|>1$.
The general statement follows below.
\begin{corollary}[Subset sum approximation \citep{subsetsum,uniExist}]\label{thm:subsetsumExtended}
Let $X_1, ..., X_{m}$  be independent bounded random variables with $|X_k| \leq B$.
Assume that each $X_k \sim X$ contains a uniform distribution with potentially different $\alpha_k > 0$ (see Definition~\ref{def:containUniform}) and $c=0$. 
Let $\epsilon, \delta \in (0,1)$ and $t \in \mathbb{N}$ with $t \geq 1$ be given. 
Then for any $z \in [-t,t]$ there exists a subset $S \subset [m]$ so that with probability at least $1-\delta$ we have $|z - \sum_{k \in S} X_k| \leq \epsilon$ if
\begin{align*}
    m \geq C\frac{\max\left\{1,\frac{t}{h}\right\}}{\min_k\{\alpha_k\}}\log\left(\frac{B}{\min\left( \frac{\delta}{\max\{1,t/h\}}, \frac{\epsilon}{\max\{t,h\}}\right)} \right).
\end{align*}
\end{corollary}

\section{Additional Statistics on Solving Subset Sum Problems}
In the 2-layers-for-one construction, we solve subset sum approximation problems based on independent random variables with distribution $X_k \sim U[-1,1] U[-1,1]$, i.e., the product of two random variables.
To be more precise, for ReLUs, the random variables are distributed as $X_k \sim U[0,1] U[-1,1]$ or $X_k \sim U[-1,0] U[-1,1]$.
Because of the symmetry of the uniform random variables, these are all identically distributed and we can just focus on the case $X_k \sim U[-1,1] U[-1,1]$.
Fig.~\ref{fig:subsetsumstats2} shows the corresponding statistics that are based on $10^5$ independent subset sum problem solutions for base sets of size $m=15$. 

A base set size of $m=15$ random variables results in an average error of $(7.6 \pm 0.3) 10^{-3}$, which is more than a magnitude higher than in the case of $X_k \sim U[0,1]$.
Furthermore, we fail to achieve an error smaller than $0.01$ in $3.5 \%$ of the cases and rely on subset sizes of $\mathds{E}(|S|) \approx 2.5$.

\begin{figure*}
\includegraphics[width=\textwidth]{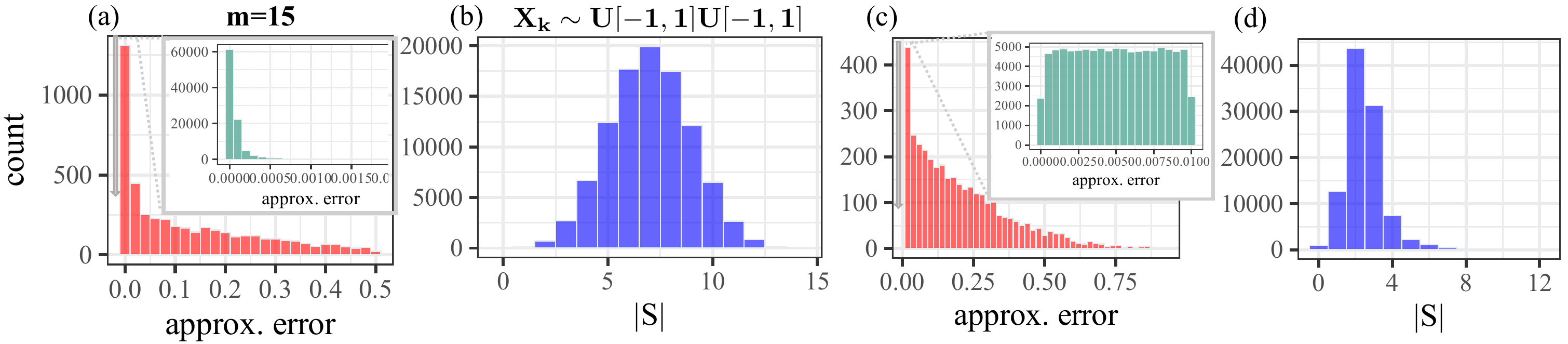}
\caption{Subset sum approximation statistics. (a) \& (c): Error of solving $10^5$ independent subset sum problems. Each problem selects $|S|$ elements out of $m=15$ independent random variables $X_k \sim U[-1,1] U[-1,1]$ to approximate a randomly drawn target $z\sim U[-1,1]$. The green histogram in the right corner focuses on the smallest errors. (b) \& (d): Size of approximating subset. (a) \& (b): Best subset selection based on exhaustive search. (c) \& (d): Smallest subset that achieves an error of maximally $\epsilon = 0.01$.}
  \label{fig:subsetsumstats2}
\end{figure*}

\section{Proofs}
We frequently use the supremum norm of tensors, which is similarly defined as a vector norm refers to the maximum absolute value of all components $\norm{X}_{\infty} := \max_{i,j,k} |x_{ijk}|$.
It should not be confused with the operator norm.

\subsection{Proof of Thm.~\ref{thm:LTexist2L}}
The following corollary covers an important step in the proof of our ($2L$)-construction, as it focuses on the approximation of a single layer by two layers in the source network.
\begin{corollary}[Layerwise approximation]\label{cor:twice}
Let a convolutional target layer $\bm{x}^{(l')}_{t,i} = \phi\left(\sum^{c_{l-1}}_{j=1} \bm{W}^{(l')}_{t,ij} * \bm{x}^{(l'-1)}_{t,j} + b^{(l')}_{t,i} \right)$ with $c_{l'}$ output, $c_{l'-1}$ input channels, and filter size $k_{l'}$,  $N_t$ nonzero parameters and activation function fulfilling Assumption~\ref{def:act} with $d=0$ be given. 
Moreover, we have a two-layer source network 
$\bm{x}^{(2l')}_{0,i} = \phi\left(\sum^{c_{1/2}}_{s=1} \bm{W}^{(2l')}_{0,is} * \phi\left(\sum^{c_{l'-1}}_{j=1} \bm{W}^{(2l'-1)}_{0,sj} * \bm{x}^{(2l'-2)}_{0,j} \right) \right)$ with parameters $\bm{W}^{(2l')}_{0} \in \mathbb{R}^{c_{l'} \times c_{1/2} \times k_{l'}}$ and $\bm{W}^{(2l'-1)}_{0} \in \mathbb{R}^{c_{1/2} \times c_{l'-1} \times 1}$.
All its tensor entries are independently uniformly distributed as $w^{(2l'-1)}_{0,sjq} \sim U[-\sigma, \sigma]$ and $w^{(2l')}_{0,isq} \sim U[-r/\sigma, r/\sigma]$. 
Then, for every ${\epsilon', \delta' \in {(0,1)}}$, with probability $1-\delta'$, there exists a sub-network $\bm{x}^{(2l')}_{\epsilon'} \subset \bm{x}^{(2l')}_{0}$ so that $\norm{\bm{x}^{(l')}_{t} - \bm{x}^{(2l')}_{\epsilon}}_{\infty} \leq \epsilon'$ if
\begin{align}
    c_{1/2} \geq C c_{l'-1} \log\left(\frac{N_t}{ \min\{\epsilon'/(3T), \delta'/2 \}} \right),
\end{align}
the input components fulfill $\norm{\bm{x}^{(l'-1)}_{t, i} - \bm{x}^{(2l'-2)}_{0,i}}_{\infty} \leq \epsilon'/(3 T N_t)$, $\norm{1- \bm{x}^{(2l'-2)}_{0, c_{l'-1}+1}}_{\infty} \leq \epsilon'/(3 T N_t)$, and $\sigma = \min\left\{1,a(\epsilon'')/2\right\}$ with
\begin{align*}
\epsilon'' =  g^{-1} \left(\frac{\epsilon'}{3 T N_t C \log\left(\frac{N_t}{\min\{\epsilon'/(3T), \delta'/2\}}\right) \frac{r}{2} } \right)
\end{align*}
for $g(\epsilon'') = \epsilon''/(a(\epsilon''))$.
\end{corollary}
\begin{proof}
The first step of our construction of the LT $\bm{x}^{(2l')}_{\epsilon'}$ is to prune the weight tensors in the first layer of the source network $\bm{W}^{(2l'-1)}_{0,js}$ to univariate form.
For each input neuron $j$, we reserve $|I_j|$ neurons in the intermediary layer $2l'-1$ with indices $I_j$ so that $w^{(2l'-1)}_{\epsilon',sj1} = w^{(2l'-1)}_{0,sj1} = \lambda_{sj}$ if $s \in I_{j}$ and $w^{(2l'-1)}_{\epsilon',sj1} = 0$ otherwise. 
After pruning, we thus have $\phi\left(\sum^{c_{l'-1}}_{j=1} \bm{W}^{(2l'-1)}_{\epsilon',sj} * \bm{x}^{(2l'-2)}_{0,j} \right) = \phi\left( \bm{W}^{(2l'-1)}_{\epsilon',sj} * \bm{x}^{(2l'-2)}_{0,j} \right) = \phi\left(\lambda_{sj} \bm{x}^{(2l'-2)}_{0,j} \right)$ for an $s \in I_j$.

Per construction of the initialization, $\sigma > 0$ is chosen small enough so that 
\begin{align}
\norm{\phi\left(\lambda_{sj} \bm{x}^{(2l'-2)}_{0,j} \right) -  \mu_{\pm}(\lambda_{sj} \bm{x}^{(2l'-2)}_{0,j} ) \lambda_{sj} \bm{x}^{(2l'-2)}_{0,j} }_{\infty} \leq \epsilon''    
\end{align}
according to Assumption~\ref{def:act}.
We achieve this, as $|\lambda_{sj} x^{(2l'-2)}_{0,jq}| \leq \sigma (1+\epsilon'/(3 T N_t)) \leq a(\epsilon'')$ with $|x^{(2l'-2)}_{0,jq}| \leq 1+\epsilon'/(3 T N_t) $, since $|x^{(l'-1)}_{t,jq}| \leq 1$ is always assumed. 

The next step of pruning is to mask some elements of the tensor in the second layer $\bm{W}^{(2l')}_{0,is}$ so that we can approximate our target parameters.
We set all parameters to zero $w^{(2l')}_{\epsilon',isq} = 0$ with the exception of $w^{(2l')}_{\epsilon',isq} = w^{(2l')}_{0,isq}$ for indices $s \in I_{ijq} \subset I_{j}$.
These index sets are chosen by solving specific subset sum approximation problems based on the following two base sets $ I_{j,+}$ and $I_{j,-}$ with $I_j = I_{j,+} \cup I_{j,-}$.
They are defined according to the incoming link weight $\lambda_{sj}$ in the first layer, i.e. $I_{j,+} = \{s \in I_j \; \mid \; \lambda_{sj} > 0 \}$ and  $I_{j,-} = \{s \in I_j \; \mid \; \lambda_{sj} < 0 \}$. 
The associated random variables $X_{s} =   w^{(2l')}_{0,isq} \lambda_{ij}/r $ are distributed as $U[0,1] U[-1,1]$ or $U[-1,0] U[-1,1]$ per construction. 
As these contain uniform distributions \citep{pensia2020optimal}, according to Thm.~\ref{thm:subsetsumExtended}, with probability $1-\delta'''$ we can find subsets $I^{\pm}_{ijq} \subset I_{j,\pm}$ so that 
\begin{align}\label{eq:sub}
    |w_{t,ijq} - \sum_{s \in I^{\pm}_{ijq}}  X_s | \leq \epsilon''',
\end{align}
if $|I_{j,\pm}| \geq C \log\left(\frac{1}{\min\{\epsilon''', \delta'''\}}\right)$.
Solving two separate problems of this form leads to an index set $I_{ijq} = I^{+}_{ijq} \cup I^{-}_{ijq}$ that defines the parameters that we keep in our LT.

We still have to approximate the target bias.
For that purpose, we have reserved a constant input tensor $\bm{x}^{(2l'-2)}_{0, c_{l'-1}+1}$ with $x^{(2l'-2)}_{0, c_{l'-1}+1, q} \approx 1$ so that $\phi(\lambda_{s (c_{l'-1}+1)}  x^{(2l'-2)}_{0, c_{l'-1}+1, q}) \approx m_{+} \lambda_{s (c_{l'-1}+1)}$.
We thus choose nonzero parameters with indices $I_{ib} \subset I_{c_{l'-1}+1}$ in the second layer that solve the subset sum approximation problem  
\begin{align}\label{eq:subb}
    |b_{t,i} - \sum_{s \in I_{ib}}  X_s | \leq \epsilon''',
\end{align}
as the random variables $X_s = w^{(2l')}_{0,is} \lambda_{s (c_{l'-1}+1)} m_+$ are distributed as $m_+ r U[0,1] U[-1,1]$ (with $| m_+ r| \leq 1$ most of the time), which also contain a uniform distribution.

After pruning the first and the second layer of the source network this way, let us analyze the error that our LT inflicts.
It follows from the Lipschitz continuity of the activation function $\phi$ and our pruning to univariate tensors in the first layer that
\begin{align}\label{eq:all2l}
\begin{split}
& \norm{\bm{x}^{(l')}_{t} - \bm{x}^{(2l')}_{\epsilon}}_{\infty} \leq T \max_i \sum_j \norm{\bm{W}^{(l')}_{t,ij} * \bm{x}^{(l'-1)}_{t,j} + b^{(l')}_{t,i}  - \sum^{c_{1/2}}_{s=1} \bm{W}^{(2l')}_{\epsilon',is} * \phi\left(\lambda_{sj}  \bm{x}^{(2l'-2)}_{0,j} \right) }_{\infty} \\
& \leq T \max_i \sum_j \norm{\bm{W}^{(l')}_{t,ij} * \bm{x}^{(l'-1)}_{t,j} + b^{(l')}_{t,i}  - \sum^{c_{1/2}}_{s=1} \bm{W}^{(2l')}_{\epsilon',is} * \left(\mu_{\pm} \left(\lambda_{sj}  \bm{x}^{(l'-1)}_{t,j} \right) \lambda_{sj} \bm{x}^{(l'-1)}_{t,j}  \right) }_{\infty} \\
& + T \max_i \sum_j \norm{\sum^{c_{1/2}}_{s=1} \bm{W}^{(2l')}_{\epsilon',is} * \left(\mu_{\pm} \left(\lambda_{sj}  \bm{x}^{(2l'-2)}_{0,j}\right)  \lambda_{sj} \left[ \bm{x}^{(2l'-2)}_{0,j} - \bm{x}^{(l'-1)}_{t,j} \right] \right) }_{\infty} \\
& + T \max_i \sum_j \norm{\sum^{c_{1/2}}_{s=1} \bm{W}^{(2l')}_{\epsilon',is} * \left[\mu_{\pm} \left(\lambda_{sj}  \bm{x}^{(2l'-2)}_{0,j} \right) \lambda_{sj}  \bm{x}^{(2l'-2)}_{0,j}  - \phi\left(\lambda_{sj}  \bm{x}^{(2l'-2)}_{0,j} \right) \right]}_{\infty} \leq \epsilon'
\end{split}
\end{align}
Note that $\mu_{\pm} \left(\lambda_{sj}  \bm{x}^{(2l'-2)}_{0,j}\right) = \mu_{\pm} \left(\lambda_{sj}  \bm{x}^{(l'-1)}_{t,j}\right)$ if $|\bm{x}^{(2l'-2)}_{t,j}| > \epsilon'/(3 T N_t)$ anyways.
Otherwise, we would prune it to zero. 
The first term concerns the subset sum approximation error, the second one the approximation of the input neurons, while the third one originates in the approximation of the activation function in the first layer.
We achieve our approximation objective if we bound each of these errors by $\epsilon'/3$
Note that the latter vanishes for ReLUs, Leaky ReLUs, or linear activation functions because our approximation would actually be exact.

Let us first bound the subset sum approximation error. 
Recall that we have assumed that $|x^{(l'-1)}_{t,jq}| \leq 1$. 
We can partition the sum over the indices $s$ to focus on the same input so that we can utilize the linearity of convolutions to derive
\begin{align}
\begin{split}
 & T \max_i \sum_j \norm{\bm{W}^{(l')}_{t,ij} * \bm{x}^{(l'-1)}_{t,j} + b^{(l')}_{t,i}  - \sum^{c_{1/2}}_{s=1} \bm{W}^{(2l')}_{\epsilon',is} * \left(\mu_{\pm} \left(\lambda_{sj}  \bm{x}^{(l'-1)}_{t,j} \right) \lambda_{sj} \right) }_{\infty} \\
& = T \max_i \sum_j \sum_q \Bigg( \Bigg| r \mu_{\pm}\left(x^{(l'-1)}_{t,jq}\right) \left(\bm{W}^{(l')}_{t,ijq} -  1/r \sum_{s \in I^{+}_{ijq}}  w^{(2l')}_{\epsilon',isq} \lambda_{sj} \right) \Bigg| + \left| r \mu_{\pm}\left(-x^{(l'-1)}_{t,jq}\right) \left(\bm{W}^{(l')}_{t,ijq} -  1/r \sum_{s \in I^{-}_{ijq}} w^{(2l')}_{\epsilon',isq} \lambda_{sj} \right) \right|\\
& +  \left|  b^{(l')}_{t,i} - m_+ \sum_{s \in I_{ib}} w^{(2l')}_{0,is} \lambda_{s (c_{l'-1}+1)}   \right| \Bigg) 
\leq T N_t \epsilon''' \leq \epsilon'/3,
\end{split}
\end{align}
where we have used Eqs.~(\ref{eq:sub}) \& (\ref{eq:subb}) and the fact that $r(\mu_{\pm}(x) + \mu_{\pm}(-x)) = 1$.
Defining $\epsilon''' = \epsilon'/(3 T N_t )$ derives our width requirement so that we can solve the associated subset sum approximation problems.
How should we choose $\delta'''$? 
In total, we have to solve less than $2 N_t$ problems.
If each is successfully solved with probability $1- \delta''' = 1-\delta'/(2 N_t)$, we can see with the help of a union bounds that we can solve all of them with probability $(1-\delta')$. 
For every input neuron, we need to prune a large enough base set with $||I_j| \geq C  \log\left(\frac{1}{\min\{\epsilon''', \delta'''\}}\right)$ resulting in a width requirement of $c_{1/2} \geq C c_{l'-1} \log\left(\frac{N_t}{ \min\{\epsilon'/(3T), \delta'/2 \}} \right)$, as was to be shown.

Second, we have to show that we can bound the approximation error of the input neurons in Eq.~(\ref{eq:all2l}).
With the help of the previous approximation and recalling that $|w_{t,ijq}| \leq 1-\epsilon'''$ so that $|w_{\epsilon',ijq}| \leq 1$, we see that 
\begin{align}\label{eq:inputerror}
\begin{split}
&  T \max_i \sum_j \norm{\sum^{c_{1/2}}_{s=1} \bm{W}^{(2l')}_{\epsilon',is} * \left(\mu_{\pm} \left(\lambda_{sj}  \bm{x}^{(2l'-2)}_{0,j}\right)  \lambda_{sj} \left[ \bm{x}^{(2l'-2)}_{0,j} - \bm{x}^{(l'-1)}_{t,j} \right] \right) }_{\infty} \\
& \leq T N_t  \max_j \max_q  | x^{(l'-1)}_{t,jq} - x^{(2l'-2)}_{0,jq} | \leq \epsilon'/3
\end{split}
\end{align}
according to our assumption on $\max_{j,q} | x^{(l'-1)}_{t,jq} - x^{(2l'-2)}_{0,jq} |$. 

Third, let us bound the activation function approximation error in Eq.~(\ref{eq:all2l}), which we control by the initialization constant $\sigma > 0$, which we can make arbitrarily small. 
\begin{align}
\begin{split}
& T \max_i \sum_j \norm{\sum^{c_{1/2}}_{s=1} \bm{W}^{(2l')}_{\epsilon',is} * \left[\mu_{\pm} \left(\lambda_{sj}  \bm{x}^{(2l'-2)}_{0,j} \right) \lambda_{sj}  \bm{x}^{(2l'-2)}_{0,j}  - \phi\left(\lambda_{sj}  \bm{x}^{(2l'-2)}_{0,j} \right) \right]}_{\infty}\\
& \leq T \sum_j \sum_q \sum_s |w^{(2l')}_{\epsilon',isq}| \epsilon'' \leq T N_t C \log\left(\frac{N_t}{\min\{\epsilon'/(3T), \delta'/2\}}\right) (r/\sigma) \epsilon''
\end{split}
\end{align}
We have to choose $\sigma = a(\epsilon'')/2$ small enough so that 
\begin{align}
  \epsilon'' \leq \frac{\epsilon'}{3 T N_t C \log\left(\frac{N_t}{\min\{\epsilon'/(3T), \delta'/2\}}\right) \frac{r}{\sigma} } = \frac{\epsilon'}{3 T N_t C \log\left(\frac{N_t}{\min\{\epsilon'/(3T), \delta'/2\}}\right) \frac{r}{2 a(\epsilon'')} }.
\end{align}
Note that we can find an appropriate $\epsilon''$ because the function $g(\epsilon'') = \frac{\epsilon''}{a(\epsilon'')}$ is invertible on a suitable interval $]0, \epsilon']$.
The reason is that $a(\epsilon'')$ is continuous and monotonically increasing in $\epsilon''$.
We can therefore define 
\begin{align}
    \epsilon'' = g^{-1} \left(\frac{\epsilon'}{3 T N_t C \log\left(\frac{N_t}{\min\{\epsilon'/(3T), \delta'/2\}}\right) \frac{r}{2} } \right).
\end{align}
With this we can conclude that the LT approximates the target network up to error $\epsilon'$. 
\end{proof}
Interestingly, note that the activation function approximation does not directly impact our width requirement.
It relies, however, on a suitable parameter initialization approach.

\paragraph{($2L_t$)-construction}
Stacking $L_t$ layers together, we can prove our LT existence theorem for the ($2L_t$)-construction.
\begin{theorem*}[LT existence ($2L_t$)-construction)]%
Assume that ${\epsilon, \delta \in {(0,1)}}$, a convolutional target network (without skip connections) $f_t(x): \mathcal{D} \subset \mathbb{R}^{c_0 \times d_0} \rightarrow \mathbb{R}^{c_L \times d_{L_t}}$ with architecture $\bar{c}_t$ of depth $L_t$ with $N_{t,l}$ nonzero parameters in Layer $l$, and a source network $f_0$ with architecture $\bar{n}_0$ of depth $L_0 = 2 L_t$ are given.
Let $\phi$ be the activation function of $f_t$ with Lipschitz constant $T$ fulfilling Assumption~\ref{def:act} with $d=0$. 
Then, with probability at least $1-\delta$, $f_{0}$ contains a subnetwork $f_{\epsilon} \subset f_{0}$ so that each output component $i$ is approximated as $\max_{\bm{x}\in \mathcal{D}} \left|f_{t,iq}(\bm{x})- f_{\epsilon',iq}(\bm{x}) \right| \leq \epsilon$ %
if for all $l' \in [L_t]$
\begin{align*}
  c_{0,2l'+1} \geq  C c_{t,l} \log\left(\frac{N_t }{\min\{\epsilon/\prod^{L_t}_{s=l}(3TN_{t,s}), \delta \} }\right), %
\end{align*}
and $n_{0,2l'} \geq n_{t,l'}+1$, and if the parameters of $f_{0}$ are initialized according to Assumption~\ref{def:initconv} with $\sigma_{2l'+1} = r/\sigma_{2l'} $ and $\sigma_{2l'} = a(\epsilon'')/2$ and $\epsilon'' = g^{-1} \left(\frac{\epsilon'}{C  N_t \log\left(\frac{N_t }{\min\{\epsilon/\prod^{L_t}_{s=l}(3TN_{t,s}), \delta \} }\right) } \right)$ for $g(\epsilon'') = \epsilon''/(a(\epsilon''))$.
\end{theorem*}
\begin{proof}
The proof is a repeated application of Corollary~\ref{cor:twice}.
Only the first layer approximation is special, as we might not have a constant tensor available among the data inputs to approximate the target biases of the first layer.
In this case, we need to modify our bias approximation by using the nonzero biases of $f_0$ in Layer $l=1$.
Instead of pruning univariate tensors in the first layer that take a constant tensor as input, the neurons that we reserve in Layer $1$ of the LT with indices $I_b$ receive completely zero weights but keep a bias term $b_{\epsilon,s} = b_{0,s}$.
The remaining steps of the proof for the first layer are identical to the proof of Corollary~\ref{cor:twice}.

It is only left to show how to adapt the $\epsilon'$ and $\delta'$ to account for multiple layer approximations.
$\delta' = \delta/N_t$ with $N_t = \sum_l N_{t,l}$ is sufficient to ensure that all subset sum approximations of all parameters are successful with probability $\delta$.
The adaptation of the error, however, needs to take into account how error propagates through different layers.
To approximate the output successfully, Corollary~\ref{cor:twice} requires that each input tensor element $x^{(L-1)}_{t,iq}$ can have an error of maximally $\epsilon/(3 T N_{t,L})$.
Repeating this argument inductively, results in an allowed error of $\epsilon_l = \frac{\epsilon}{(3T)^{L_t-l+1} \prod^{L_t}_{s=l} N_{t,s}}$ in the approximation of the target layer $l$. 
\end{proof}
Note that with stricter assumptions on the target parameters, we could also obtain a more favorable scaling of the error with the number of nonzero parameters.

\subsection{Proof of Thm.~\ref{thm:LTexistfull} ($L_t+1$)-construction)}

\begin{theorem*}[LT existence ($L_t+1$)-construction)]%
Assume that ${\epsilon, \delta \in {(0,1)}}$, a convolutional target network (possibly with skip connections) $f_t(x): \mathcal{D} \subset \mathbb{R}^{c_0 \times d_0} \rightarrow \mathbb{R}^{c_L \times d_{L_t}}$ with architecture $\bar{c}_t$ of depth $L_t$, and a source network $f_0$ with architecture $\bar{n}_0$ of depth $L_0 = L_t+1$ are given.
Let $\phi$ be the activation function of $f_t$ and $f_0$ with Lipschitz constant $T$. 
Let further $\phi_0$ be the activation function of $f_0$ in the first layer fulfilling Assumption~\ref{def:act} with $d=0$. 
Define the number $N_{l}$ of effective nonzero parameters in Layer $l$ as $N_{l} = N_{w,l} + N_{m,l}$.
Then, with probability at least $1-\delta$, $f_{0}$ contains a subnetwork $f_{\epsilon} \subset f_{0}$ so that each output component is approximated as $\max_{\bm{x}\in \mathcal{D}} \left|f_{t,iq}(\bm{x})- f_{\epsilon',iq}(\bm{x}) \right| \leq \epsilon$ %
if for all $l \in [L_t]$
\begin{align*}
  c_{0,l+1} \geq  C c_{t,l} \log\left(\frac{1 }{\min\{\epsilon_l, \rho \delta/N_l \} }\right), %
\end{align*}
and $n_{0, 1} \geq C c_{t,0} \log\left(\frac{1 }{\min\{\epsilon_1, \delta \rho \} }\right)$ with $\rho = C N^{1+\gamma}_l \log(1/\min\{ \min_l \epsilon_{l}, \delta \})$ for any $\gamma > 0$.
$\epsilon_l = \frac{\epsilon}{2 T N_{w,l} \prod^L_{s=l+1} (2 (T N_{w,s} + N_{m,s}))}$.
Additionally, we require that the parameters of $f_{0}$ are initialized according to Assumption~\ref{def:initconv} with $\sigma_l = 1$ for $l > 2$, $\sigma_{1} = r/\sigma_{2} $ and $\sigma_{2} = a(\epsilon'')/2$ and $\epsilon'' = g^{-1} \left(\frac{\epsilon'}{C  N_l \log\left(\frac{N_l }{\min\{\epsilon_l, \delta \} }\right) } \right)$ for $g(\epsilon'') = \epsilon''/(a(\epsilon''))$.
\end{theorem*}

\begin{proof}
In contrast to Thm.~\ref{thm:LTexist2L}, in the approximation of the target layers $l > 1$, we do not need to approximate the activation function locally as a Leaky ReLU.
Thus, we save the approximation error and the separate approximation of positive and negative parts.
Moreover, we can use smaller base sets to solve subset sum approximation problems because the random variables are distributed as $X_k \sim U[-1,1]$ instead of $X_k \sim U[0,1]  U[-1,1]$. 

Another advantage of this construction is that skip connections can be naturally integrated.
As a consequence, we also have to consider the error that we inflict by pruning or just representing skip connections.
Let us regard the error at Layer $l$ and denote with $\epsilon_l$ the maximal error of a parameter approximation.
Similar to before, we have 
\begin{align}
\begin{split}
     \norm{x^{(l)}_{t} -  x^{(l+1)}_{0} }_{\infty} & \leq T N_{w,l} \epsilon_l + T N_{w,l}  \norm{x^{(l-1)}_{t} -  x^{(l)}_{0} }_{\infty} + N_{m,l} \max_{s \leq l-1}  \norm{x^{(s)}_{t} -  x^{(s+1)}_{0} }_{\infty} \\
    & \leq  T N_{w,l} \epsilon_l + (TN_{w,l} + N_{m,l} ) \max_{s \leq l-1}  \norm{x^{(s)}_{t} -  x^{(s+1)}_{0} }_{\infty}. 
\end{split}
\end{align}
In fact, it also follows that 
\begin{align}
\begin{split}
   \max_{s \leq l}   \norm{x^{(s)}_{t} -  x^{(s+1)}_{0} }_{\infty} & \leq  T N_{w,l} \epsilon_l + (TN_{w,l} + N_{m,l} ) \max_{s \leq l-1}  \norm{x^{(s)}_{t} -  x^{(s+1)}_{0} }_{\infty}. 
\end{split}
\end{align}
The error of the last layer can therefore be bounded by $\epsilon$ if we ensure that $\epsilon_L = \epsilon/(2 T N_{w,l})$ and
$\max_{s \leq L-1}  \norm{x^{(s)}_{t} -  x^{(s+1)}_{0} }_{\infty} \leq \epsilon/(2(TN_{w,L} + N_{m,L}))$. 
We can thus derive the error by propagating it from Layer $l$ to $l-1$.
This leads to the definition $\epsilon_l = \frac{\epsilon}{2 T N_{w,l} \prod^L_{s=l+1} (2(TN_{w,L} + N_{m,L}))}$.
 
In addition, we need to investigate how many more subset sum approximations we have to solve in this construction.
The argument is very similar to the one for fully-connected networks \citep{depthexist}.
For completeness, we repeat it here.

$\delta$ is modified by $\rho \geq \rho' = \sum^L_{l=1} \rho'_l$, where $\rho'$ counts the increased number of required subset sum approximation problems to approximate the $L$ target layers with our lottery ticket and $\rho_l$ counts the same number just for Layer $l$. 

For each non-zero parameter, we will need two solve at least one subset sum approximation problem or sometimes two in case of the first target layer.
We denote the number of non-zero parameters in Layer $l$ as $N_{l}$. 
Thus, if our target network is dense without skip connections and all parameters are nonzero, we have $N_l= c_{t,l} c_{t,l-1} k_{t,l}$ and in total $N_t = \sum^L_{l=1} c_{t,l} (c_{t,l-1} k_{t,l} + 1)$.

Let us start with counting the number $\rho'_L$ of required subset sum approximation problems in the last layer because it determines how many neurons we need in the previous layer. 
This in turn defines how many subset sum approximation problems we have to solve to construct this previous layer.

The last layer requires us to solve exactly $\rho'_L = N_L$ subset sum problems, which can be solved successfully with high probability if $c_{0,L-1} \geq C c_{t,L-1} \log(1/\min\{\epsilon_L, \delta/\rho'\})$.
We will only need to construct a subset of these neurons with the help of Layer $L-2$, i.e., exactly the neurons that are used in the lottery ticket. 
If $c_{0,L-1}$ is large, this might require only $2-3$ neurons per parameter.
For simplicity, however, we bound this number by the total number of available channels.
To reconstruct one set of channels, we need approximate $N_{L-1}$ parameters.
As we have to maximally construct $C\log(1/\min\{\epsilon_L, \delta/\rho'\})$ sets of these channels, we can bound $\rho'_{L-1} \leq C N_{L-1} \log(1/\min\{\epsilon_L, \delta/\rho'\}$.

Note that we can solve all of these subset sum approximation problems with the help of $c_{t,L-2} \geq C N_{L-2} \log(1/\min\{\epsilon_{L-1}, \delta/\rho'\}$ of neurons and this number does not scale by the fact that we have to construct not only $c_{t,L-1}$ channels but a number that is increased by a logarithmic factor.
The higher number of required neuron approximations only affects the number of required subset sum approximation problems and thus the needed success probability of each parameter approximation via $\rho$.

Repeating the same argument for every layer, we derive $\rho'_{l} \leq C N_{l} \log(1/\min\{\epsilon_{l+1}, \delta/\rho'\}$, which could also be shown formally by induction.
In total we thus find $\rho' = \sum^L_{l=1} \rho'_l \leq C N_{l}\log(1/\min\{ \min_l \epsilon_{l}, \delta/\rho' \}) \leq C N_t \log(1/\min\{ \min_l \epsilon_{l}, \delta/\rho \})$.
A $\rho$ that fulfills $ \rho = C N_t \log(1/\min\{ \min_l \epsilon_{l}, \delta/\rho \})$ would therefore be sufficient to prove our claim.
$\rho = C N^{1+\gamma}_t \log(1/\min\{\epsilon, \delta \})$ for any $\gamma > 0$ works, as $C N^{\gamma}_t \geq \log(N_t)$.
\end{proof}

\section{Activation functions with $d \neq 0$}
As explained for fully-connected networks by \citep{depthexist}, our derivations also apply to activation functions with $\phi(0) = d \neq 0$ if we initialize the parameters in our source network with the 'looks-linear' initialization \citep{dyniso}.

\section{Strides}
The main objective of pruning the first layer in Corollary~\ref{cor:twice} is to create multiple versions of the input tensors.
This can be achieved by pruning 1-dimensional filters with stride 1.
If the stride is higher and the filter dimensions is big enough so that filter windows overlap, we can always prune the available filter down to an univariate one - with the given stride $s$.
Yet, with such a filter, if the stride is $s>1$, we will not create a complete version of an input filter.
If we multiply our width requirement by the stride $s$, we can still reconstruct it by adding partial input filter versions.
This can be achieved by pruning a filter in the source network at different positions.
For each additional position, we need another univariate filter, which explains our increased width requirement.

\end{document}